%% file: arxiv.tex
\documentclass[onefignum,onetabnum]{siamart251216}
\usepackage{amssymb}
\hypersetup{colorlinks=true,linkcolor=red,citecolor=blue,urlcolor=cyan}
\usepackage{bm}
\usepackage{booktabs}
\usepackage{graphicx}
\usepackage{array}

\newcommand{\RR}{\mathbb{R}}
\newcommand{\calD}{\mathcal{D}}
\newcommand{\calE}{\mathcal{E}}
\newcommand{\calF}{\mathcal{F}}

\newcommand{\calL}{\mathcal{L}}
\newcommand{\calN}{\mathcal{N}}

\newcommand{\calV}{\mathcal{V}}
\newcommand{\calW}{\mathcal{W}}
\newcommand{\calZ}{\mathcal{Z}}
\newcommand{\calM}{\mathcal{M}}

\newcommand{\calI}{\mathcal{I}}
\newcommand{\calS}{\mathcal{S}}

\newcommand{\Mset}[2]{\calM_{#1}(#2)}

\newsiamremark{remark}{Remark}

\AtBeginDocument{%
  \crefformat{section}{\S#2#1#3}%
  \crefformat{subsection}{\S#2#1#3}%
  \crefformat{subsubsection}{\S#2#1#3}%
  \Crefformat{section}{\S#2#1#3}%
  \Crefformat{subsection}{\S#2#1#3}%
  \Crefformat{subsubsection}{\S#2#1#3}%
}

\headers{Completeness Theory for GNN Interatomic Potentials}{P.~Ming and H.~Wang}

\title{Why Multi-Layer Message Passing Works:
  Completeness Theory for Graph Neural Network
  Interatomic Potentials
}

\author{Pingbing Ming\thanks{SKLMS, Institute of Computational Mathematics and
  Scientific/Engineering Computing, AMSS, Chinese Academy of Sciences,
  No.~55, East Road Zhong-Guan-Cun, Beijing 100190, China
  (\email{mpb@lsec.cc.ac.cn}).}
  \and Han Wang\thanks{National Key Laboratory of Computational Physics,
  Institute of Applied Physics and Computational Mathematics,
  Fenghao East Road 2, Beijing 100094, P.R.~China, and
  HEDPS, CAPT, College of Engineering, Peking University,
  Beijing 100871, P.R.~China
  (\email{wang\_han@iapcm.ac.cn}).}}

\usepackage{docmute}
\newcommand{\arxivbuild}{}

\begin{document}

\input{main}

\clearpage
\def\siamprelabel{SM}
\setcounter{section}{0}
\begin{center}
{\bfseries\large Supplementary Materials}
\end{center}

\input{supplement}

\end{document}

%% file: main.tex
\maketitle

\begin{abstract}
We prove that the Hypergraph Neural Network, an
invariant architecture with 3-body message
passing, is a universal approximator for potential energy
surfaces.  Our main contribution is a \emph{multi-layer
completeness theory}. We show that $L$ layers of
message passing on sparse, cutoff-based graphs achieve the
same representational power as having access to the full
$L$-hop neighborhood, provided the configurations are generic, satisfy an overlap condition 
and a 
connectivity condition.
This provides the first rigorous justification for the
common practice of using multi-layer message passing
with a per-layer cutoff smaller than the physical
interaction range, the setting used by virtually all
practical graph neural network based
machine-learned interatomic potentials. As immediate consequences, we show that both DPA3 and CHGNet architectures inherit universal approximation.
\end{abstract}



\section{Introduction}
\label{sec:intro}

Machine-learned interatomic potentials
(MLIPs)~\cite{behler2007generalized,bartok2010gaussian,
zhang2018deeppotse} have become an essential tool in
computational chemistry and materials science, enabling
molecular dynamics simulations that achieve near-quantum-mechanical accuracy at a fraction of the computational cost.
Applications range from drug
discovery~\cite{wang2024drug} and
catalysis~\cite{lan2023adsorbml} to materials
design~\cite{chen2022m3gnet} and phase diagram
exploration~\cite{wen2024alloy}. Recent advances in atomistic foundation models~\cite{zhang2025dpa3,batatia2024mace_mp,
merchant2023gnome,fmreview2025} have further expanded the
scope of MLIPs to universal potentials covering much of the
periodic table. This raises a fundamental question: can a multi-layer GNN-based MLIP approximate an arbitrary potential energy surface (PES) while rigorously preserving the required physical symmetries and continuity constraints? The answer to this question is of central importance, as it determines whether the predictive accuracy of such models is inherently limited by the architectural design, or whether it ultimately depends only on the training data and optimization procedure.

A closely related concept is \emph{completeness}.  A model is said to be complete if it produces distinct outputs for any two local environments that are not related by symmetry, that is, two
environments yield the same output only if they are related
by a rotation or reflection combined with a permutation of atoms of the
same species. Completeness is a necessary
condition for universal approximation,
yet whether it is also
sufficient is a separate question that we shall address in this
work. The concept was introduced by
Bart{\'o}k et~al.~\cite{bartok2013soap} in the context of
invariant representations, and brought into sharp focus by
Pozdnyakov et~al.~\cite{pozdnyakov2020incompleteness}, who
demonstrated that the widely used 3-body representations, such as SOAP power spectra, are \emph{incomplete}, in that there exist geometrically distinct environments that share identical sets of pairwise distances and triplet angles. Incompleteness thus imposes a fundamental ceiling on approximation accuracy that cannot be overcome, regardless of the training data.

Further incompleteness results have since revealed the scope
of the problem.  Pozdnyakov and
Ceriotti~\cite{pozdnyakov2022incompleteness} showed that
distance-only GNNs, such as
SchNet~\cite{schutt2018schnet}, are incomplete for
3D point clouds even with arbitrary depth and
nonlinearity, so angular information is not merely helpful but
\emph{necessary}.  Cen et~al.~\cite{cen2024hegnn} showed that
equivariant GNNs with fixed tensor degree~$l$ degenerate to
zero on symmetric structures for specific values of~$l$,
e.g., $l = 1$ fails on all symmetric graphs, demonstrating
that low-degree equivariant features are also insufficient.

Several approaches have been developed to achieve
completeness.  The Atomic Cluster Expansion
(ACE)~\cite{drautz2019ace,dusson2022ace} and Moment Tensor
Potentials (MTP)~\cite{shapeev2016mtp} construct systematically
improvable polynomial bases for invariant functions.  These are
\emph{linear} models
for which basis
density directly yields universal approximation.  Equivariant message-passing architectures such as
Tensor Field Networks~\cite{thomas2018tfn},
NequIP~\cite{nequip2022},
PaiNN~\cite{schutt2021painn},
SEGNN~\cite{brandstetter2022segnn}, and
Equivariant Transformers~\cite{tholke2022equivariant}
use $E(3)$-equivariant convolutions
to propagate vector and tensor features, achieving high data
efficiency.
Formal expressiveness guarantees for this line of work are available in the fully connected setting.
Dym and Maron~\cite{dym2021universality} proved that Tensor Field Networks are universal for equivariant functions on fully connected point clouds.
MACE~\cite{batatia2022mace} constructs the ACE basis within the GNN framework and thereby inherits completeness from ACE theory.
A different approach processes \emph{invariant} scalar features, namely distances and bond angles, with nonlinear networks.
Models in this class include
DimeNet~\cite{gasteiger2020dimenet},
GemNet~\cite{gasteiger2021gemnet},
SphereNet~\cite{liu2022spherenet},
ALIGNN~\cite{choudhary2021alignn},
M3GNet~\cite{chen2022m3gnet},
CHGNet~\cite{deng2023chgnet}, and
DPA3~\cite{zhang2025dpa3}.
Finally, coordinate-frame methods construct local reference
frames from pairs of neighbors and apply nonlinear functions
before symmetrization:
Nigam et~al.~\cite{nigam2024completeness} proposed
3-center-1-neighbor representations, and Pozdnyakov and
Ceriotti~\cite{pozdnyakov2023smooth} introduced the
Equivariant Coordinate System Ensemble (ECSE).

On the theoretical side, several foundational results underpin
the analysis of model expressiveness. The DeepSets
universality theorem~\cite{zaheer2017deepsets,
wagstaff2019limitations,han2019universal} established that $\phi \circ (\sum_j g)$ is a universal approximator for
permutation-invariant functions. Villar
et~al.~\cite{villar2021scalars} proved that any function
invariant under translations, rotations, reflections, and
permutations can be expressed as a permutation-invariant
function of the Gram matrix of displacement vectors, a result known as the
``scalars are universal'' theorem.  Li
et~al.~\cite{li2024completeness} demonstrated that several invariant
GNN architectures, such as DimeNet, SphereNet and GemNet, achieve
$E(3)$-completeness on fully connected graphs, and Hordan et~al.~\cite{hordan2024complete} constructed complete invariant networks of polynomial size in the same setting.
Joshi et~al.~\cite{joshi2023expressive} introduced the
Geometric Weisfeiler--Leman test, covering both
invariant and equivariant models, and proved an equivalence
between pairwise discrimination within a model class and universal approximation.
For sparse graphs, Sverdlov and Dym~\cite{sverdlov2025expressive} proved that maximally expressive message passing networks with rotation equivariant features generically separate all connected geometric graphs, whereas networks restricted to pairwise distance features generically succeed exactly on generically globally rigid graphs.

Despite this progress, significant gaps remain.
The completeness results of Villar et~al.~\cite{villar2021scalars}, Li et~al.~\cite{li2024completeness}, Delle Rose et~al.~\cite{dellerose2023wl}, and Hordan et~al.~\cite{hordan2024complete} are established only for \emph{fully connected} graphs.
For the sparse cutoff-based graphs employed in practice, the guarantees of Sverdlov and Dym~\cite{sverdlov2025expressive} concern generic separation, with the positive result built on equivariant vector features.
Villar et~al.\
characterize \emph{what} form invariant functions must take, but do not show any specific architecture can actually approximate them.
Nigam et~al.~\cite{nigam2024completeness} identify the correct geometric insight, namely that three-center features with nonlinearity applied before summation achieve completeness, yet their argument is presented as a proof sketch rather than a rigorous derivation.
The discrimination--UAT equivalence of Joshi et~al.~\cite{joshi2023expressive} operates on the \emph{function class}, requiring only that for each pair of environments some function in the class separates them.
This is conceptually different from completeness of a \emph{single} model with fixed weights.
Although separation by a single model is achievable~\cite{dym2024low,blumsmith2025galois}, it does not provide the extensive and local structure of an interatomic potential.
Moreover, to the best of our knowledge, no existing analysis concerns the approximation of a PES with a prescribed physical interaction range. In practice, multi-layer message passing extends the receptive field from the per-layer cutoff $r_c$ to a larger physical interaction range
$R_c>r_c$, and the impact of this extended range on completeness has not yet been addressed.

We close these gaps by introducing the Hypergraph Neural
Network (HGNN), an invariant 3-body architecture, and proving
the following multi-layer completeness theorem for HGNN:

\smallskip
\noindent\textbf{Main result (see \cref{thm:uat}).}
\emph{On generic configurations satisfying an overlap and a connectivity condition, and for $L$
large enough that the $L$-hop neighborhood covers the physical
interaction range~$R_c$, the $L$-layer HGNN is a universal approximator for the PES at range~$R_c$.
}
\smallskip

The above result extends completeness from the
per-layer cutoff~$r_c$ to the physical range~$R_c$ via
multi-layer message passing, a setting not addressed in prior work.  The proof operates on \emph{sparse} cutoff-based graphs,
provides a rigorous continuous approximation guarantee for the HGNN, and characterizes a \emph{single} fixed-weight representation rather than a function class.
Beyond this, the HGNN serves as a \emph{reference theoretical architecture}: any practical architecture that can simulate the HGNN inherits its universal approximation property.
We demonstrate this implication explicitly with DPA3 and CHGNet, i.e.,~\cref{cor:dpa3,cor:chgnet}.
The proofs further reveal one key architectural requirement, that message functions must be Multi-Layer Perceptrons (MLPs)~\cite{cybenko1989approx,hornik1989mlp}, rather than single linear layers, to ensure the desired expressiveness.

The paper is organized as follows.
\cref{sec:notation} introduces the PES properties,
symmetry framework, and the GNN-based interatomic potential
architectures considered in this work, namely, HGNN, ALIGNN, CHGNet,
DPA3. \cref{sec:completeness} proves the completeness--UAT
equivalence and establishes the existence of complete
representations via Gram matrix reconstruction.
\cref{sec:hgnn-approx} proves that the HGNN can
approximate complete representations.
\cref{sec:uat} states the UAT and derives corollaries for DPA3 and CHGNet, with the simulation proofs given in the supplementary materials.
\cref{sec:discussion} discusses implications,
assumptions, limitations, and open problems.
\section{Graph Neural Network Interatomic Potentials}
\label{sec:notation}

\subsection{Potential energy surface and target function class}

Consider a system of $N$ atoms with types $z_i \in \calZ = \{1,
\ldots, T\}$ and pairwise distinct positions $\bm{r}_i \in \RR^3$.
Write $B_r = \{\bm{r} \in \RR^3 : |\bm{r}| < r\}$ for the open ball of radius~$r$.
We assume the density is bounded, in that the number of atoms in any ball of fixed radius admits a uniform bound.
We write $M$ for such a bound at the largest cutoff appearing below.

\begin{definition}[Local environment]
\label{def:local-env}
The \emph{local environment} of atom~$j$ within cutoff~$r$ is
\begin{equation}\label{eq:local-env}
  \calD_j^{(r)} = \bigl(z_j,\; \{(\Delta\bm{r}_{jk},\, z_k)
    : k \neq j,\; \Delta\bm{r}_{jk} \in B_r\}\bigr),
  \qquad \Delta\bm{r}_{jk} = \bm{r}_k - \bm{r}_j,
\end{equation}
where the neighbor set is unordered (a multiset).  That is,
$\calD_j^{(r)}$ consists of the center type~$z_j$ together
with the displacement vectors and types of all neighbors
within~$r$.
We write $\calD_j \equiv \calD_j^{(r_c)}$ when the cutoff is
the default~$r_c$.
\end{definition}

A PES is a map
$E : \bigcup_{N=1}^{\infty} (\RR^3 \times \calZ)^N \to \RR$
subject to four structural properties:

\paragraph{(P1) Symmetry}
$E$ is invariant under translations, rotations/reflections,
and permutation of same-type atoms:
\begin{align}
  E(\{\bm{r}_i + \bm{t},\, z_i\}) &= E(\{\bm{r}_i,\, z_i\})
    &&\forall\, \bm{t} \in \RR^3,
    \label{eq:trans-inv} \\
  E(\{R\bm{r}_i,\, z_i\}) &= E(\{\bm{r}_i,\, z_i\})
    &&\forall\, R \in O(3),
    \label{eq:rot-inv} \\
  E(\{\bm{r}_{\pi(i)},\, z_i\})
    &= E(\{\bm{r}_i,\, z_i\})
    &&\forall\, \pi \in S_{N_1} \times \cdots \times S_{N_T},
    \label{eq:perm-inv}
\end{align}
where $N_t = |\{i : z_i = t\}|$, $S_n$ denotes the symmetric
group on $n$ elements, and $\pi$ permutes only positions of
same-type atoms.  Each atom's type is an intrinsic attribute that travels with it under
permutation.  Swapping two atoms exchanges their positions while
each retains its type, so only same-type swaps leave the
configuration unchanged.

\paragraph{(P2) Extensivity}
$E$ decomposes as a sum of atomic contributions:
\begin{equation}\label{eq:extensive}
  E(\{\bm{r}_i,\, z_i\}_{i=1}^N) =
    \sum_{i=1}^{N} \varepsilon_i.
\end{equation}

\paragraph{(P3) Locality}
There exists a physical interaction range $R_c > 0$ such that
$\varepsilon_i$ depends only on atoms within~$R_c$ of
atom~$i$:
$\varepsilon_i = \varepsilon(\calD_i^{(R_c)})$.

\paragraph{(P4) Smoothness}
$\varepsilon$ is at least $C^k$ ($k \ge 2$) with respect to the neighbor positions, including at the interaction-range boundary, where neighbors enter and leave the environment.
Since $\varepsilon$ is $C^k$, the forces are $C^{k-1}$ and the Hessian is $C^{k-2}$.
Requiring $k \ge 2$ therefore makes all three continuous.

The extensivity and locality properties reduce the
problem from a function on a $3N$-dimensional configuration
space to a function on the space of local environments, which
has bounded dimensionality.  We now formalize this space.

\begin{definition}[Environment space]
\label{def:env-space}
For a fixed cutoff $r > 0$ and a type signature $(z_0, \bm{z})$ with $z_0 \in \calZ$, $\bm{z} \in \calZ^n_{\le} = \{(z_1, \ldots, z_n) \in \calZ^n : z_1 \le \cdots \le z_n\}$, and $n \le M$, the \emph{stratum} $\calE_{z_0, \bm{z}}^{(r)}$ is
the set of ordered lists of the neighbors of a local environment within cutoff~$r$ (\cref{def:local-env}) whose center has type~$z_0$, arranged so that the $k$-th
neighbor has type~$z_k$.
Since the types are fixed by the signature, such a list is determined by its tuple of displacement vectors, written
\begin{equation}\label{eq:stratum}
  \calD = (\Delta\bm{r}_1, \ldots, \Delta\bm{r}_n)
  \in \bigl(B_{r} \setminus \{0\}\bigr)^n
    \;\subset\; \RR^{3n}.
\end{equation}
Every such tuple occurs, so the stratum is an open subset of~$\RR^{3n}$ and inherits its Euclidean metric, standard topology, and $C^\infty$ differentiable structure.
The type signature $(z_0, \bm{z})$ is index data rather than set data, so strata with different signatures are disjoint even when they have equal dimension.
The \emph{environment
space} with cutoff~$r$ is the disjoint union over all center
types and neighbor compositions:
\begin{equation}\label{eq:env-space}
  \calE^{(r)} = \bigsqcup_{z_0 \in \calZ}\;
    \bigsqcup_{n=0}^{M}\;
    \bigsqcup_{\bm{z} \in \calZ^n_{\le}}\;
    \calE_{z_0, \bm{z}}^{(r)}\,.
\end{equation}
We write $\calE \equiv \calE^{(r_c)}$ and
$\calE_{z_0, \bm{z}} \equiv \calE_{z_0, \bm{z}}^{(r_c)}$
when the cutoff is the default~$r_c$.
\end{definition}

The Euclidean metric on each stratum is
\begin{equation}\label{eq:stratum-metric}
  d(\calD, \calD')
  = \Bigl(\sum_{k=1}^{n}
    \|\Delta\bm{r}_{k} - \Delta\bm{r}_{k}'\|^2
  \Bigr)^{1/2},
  \qquad
  \calD, \calD' \in \calE_{z_0, \bm{z}}^{(r)}.
\end{equation}
Two tuples in the same stratum list the same local environment precisely when they differ by a type-preserving permutation of the neighbors, which is part of the symmetry group of \cref{def:sym-equiv}.
We refer to elements of~$\calE^{(r)}$ simply as local environments.
Since the type tuple is nondecreasing, every local environment has all of its representatives in a single stratum.
The mathematical structure of~$\calE$ is
stratum-wise, in that continuity, compactness, and $C^k$-smoothness are
defined with respect to the Euclidean metric and differentiable
structure on each stratum
$\calE_{z_0, \bm{z}}^{(r)} \subset \RR^{3n}$.  Globally,
$\calE^{(r)}$ is a disjoint union of strata of different
dimensions, while it is not a linear space.

The \emph{symmetry group} acting on a stratum of~$\calE$ with
$n$ neighbors of types $\bm z$ is
\begin{equation}\label{eq:sym-group}
  G = O(3) \times \bigl(S_{n_1} \times \cdots \times S_{n_T}\bigr),
\end{equation}
where $S_{n_t}$ permutes the $n_t$ neighbors of type~$t$
with $\sum_t n_t = n$.  We write
$\Gamma = S_{n_1} \times \cdots \times S_{n_T}$ for the
permutation part.  Only type-preserving permutations appear,
since atom types are intrinsic attributes, cf.~(P1).
The action of $(R, \pi) \in G$ on a stratum element
$\calD = (\Delta\bm{r}_1, \ldots, \Delta\bm{r}_n) \in
\calE_{z_0, \bm{z}}$ is
\begin{equation}\label{eq:group-action}
  (R, \pi) \cdot
  (\Delta\bm{r}_1, \ldots, \Delta\bm{r}_n)
  =
  (R\,\Delta\bm{r}_{\pi^{-1}(1)}, \ldots,
   R\,\Delta\bm{r}_{\pi^{-1}(n)}).
\end{equation}

\begin{definition}[Symmetry equivalence]
\label{def:sym-equiv}
Two environments $\calD, \calD'$, possibly of different
atoms or different configurations, are \emph{equivalent}, written $\calD \sim \calD'$, if they lie in the same stratum and
there exists $(R, \pi) \in G$ such that
$(R, \pi) \cdot \calD = \calD'$, that is,
\begin{equation}\label{eq:sym-equiv}
  R\,\Delta\bm{r}_{k} = \Delta\bm{r}_{\pi(k)}'
  \qquad \text{for all } k.
\end{equation}
\end{definition}

Write $d_{jk} = |\bm{r}_k - \bm{r}_j|$ for the distance between atoms $j$ and~$k$, and $\theta_{ijk}$ for the bond angle at~$j$ between neighbors~$i$ and~$k$.

Rotations and reflections change the displacement vectors but leave their mutual inner products unchanged.
Collecting those inner products therefore retains exactly the $O(3)$ invariants of the geometry. Two environments have the same inner products if and only if they lie in the same $O(3)$ orbit.
\begin{definition}[Gram matrix]
\label{def:gram}
The Gram matrix of a local environment is
\begin{equation}\label{eq:gram}
  G_{ik} = \Delta\bm{r}_i \cdot \Delta\bm{r}_k
  = d_i\, d_k\, \cos\theta_{ijk},
  \qquad i, k = 1, \ldots, n,
\end{equation}
where $d_i = |\Delta\bm{r}_i|$ and $\theta_{ijk}$ is the angle between $\Delta\bm{r}_i$ and $\Delta\bm{r}_k$.
\end{definition}

Within a stratum~$\calE_{z_0, \bm{z}}$, the type tuple~$\bm{z}$ is
fixed, so the Gram matrix~$G$ alone encodes all geometric
information.  The type-preserving permutation
group~$\Gamma$ is determined by the stratum.

For any atom~$k$, we write $G^{(k)}$ for the Gram matrix of atom~$k$'s own local environment.
Its entries are $G^{(k)}_{ab} = \Delta\bm{r}_a^{(k)} \cdot \Delta\bm{r}_b^{(k)}$ with $\Delta\bm{r}_a^{(k)} = \bm{r}_a - \bm{r}_k$.
\begin{definition}[Invariant function]
\label{def:inv-func}
A function $\varepsilon : \calE^{(r)} \to \RR$ is \emph{invariant}
under the symmetry group~$G$ if
$\varepsilon(\calD) = \varepsilon(\calD')$ whenever
$\calD \sim \calD'$.
\end{definition}
\begin{definition}[Target function space]
\label{def:target-space}
Denote by $\calF_{R_c}^k$ the set of functions
$\varepsilon : \calE^{(R_c)} \to \RR$ satisfying:
\begin{enumerate}
\item $G$-invariance;
\item $C^k$-smoothness on each stratum: the restriction of $\varepsilon$ to each $\calE_{z_0, \bm{z}}^{(R_c)} \subset \RR^{3n}$ is a $C^k$ function with respect to the Euclidean differentiable structure;
\item cross-stratum $C^k$ matching: if $\calD = (\Delta\bm{r}_1, \ldots, \Delta\bm{r}_n)$ lies in $\calE_{z_0, \bm{z}}^{(R_c)}$ and $\calD^-$ denotes the environment with the $m$-th neighbor deleted, then $\varepsilon(\calD) \to \varepsilon(\calD^-)$ together with all partial derivatives of order $\le k$ in the retained coordinates, while all derivatives involving $\Delta\bm{r}_m$ tend to zero, as $|\Delta\bm{r}_m| \to R_c$.
\end{enumerate}
\end{definition}

Condition~3 of \cref{def:target-space} ensures that $\varepsilon$ is well defined and is $C^k$ across the strata of different dimension that $\calE^{(R_c)}$ comprises. It is met, for instance, by $\varepsilon(\calD) = \sum_{i=1}^{n} s(|\Delta\bm{r}_i|)\, u(z_0, z_i, |\Delta\bm{r}_i|)$ with $u$ a $C^k$ function and $s : [0, R_c] \to [0,1]$ satisfying $s^{(m)}(R_c) = 0$ for $0 \le m \le k$.

\subsection{Machine learning interatomic potentials}
\label{sec:hgnn}

An MLIP approximates the
local energy $\varepsilon(\calD_j^{(R_c)})$ by
\begin{equation}\label{eq:mlip}
  \varepsilon_\theta(\calD_j^{(R_c)}) = f_\theta\bigl(\Phi_\theta(\calD_j^{(R_c)})\bigr),
\end{equation}
where $\Phi_\theta : \calE^{(R_c)} \to \RR^d$ is a \emph{representation}
that maps the local environment to a $d$-dimensional feature
vector, and $f_\theta : \RR^d \to \RR$ is a readout MLP
that maps the representation to an atomic energy
contribution.  Both carry learnable parameters, collectively denoted~$\theta$, and the total
energy is
\begin{equation}\label{eq:mlip-decomp}
  E_\theta(\{\bm{r}_i,\, z_i\}_{i=1}^N)
  = \sum_{i=1}^{N} \varepsilon_\theta(\calD_i^{(R_c)}).
\end{equation}
The ansatz builds (P2) and (P3) into the architecture.
The total energy is a sum of atomic contributions, each depending only on the $R_c$-environment.
The representation must be $G$-invariant so that (P1) holds.
Smoothness of $\Phi_\theta$ and~$f_\theta$, including at the cutoff boundary, yields (P4).
The quality of the MLIP hinges on the
expressiveness of the representation.

Two architectural families instantiate this ansatz. For
\emph{linear} models such as ACE~\cite{drautz2019ace,dusson2022ace}
and MTP~\cite{shapeev2016mtp}, $\Phi_\theta$ provides a complete
polynomial basis $\{B_\alpha\}$ and $f_\theta = \sum_\alpha c_\alpha B_\alpha$
is a linear combination.
Universal approximation then follows from the density of the basis.
For \emph{nonlinear} models, $f_\theta$ is an MLP and the relevant criterion is not
basis density but \emph{injectivity} of
$\Phi_\theta$ on symmetry orbits.  \cref{sec:gnn-arch} describes specific GNN-based realizations of $\Phi_\theta$.
\begin{figure}[t]
\centering
\includegraphics[width=\textwidth]{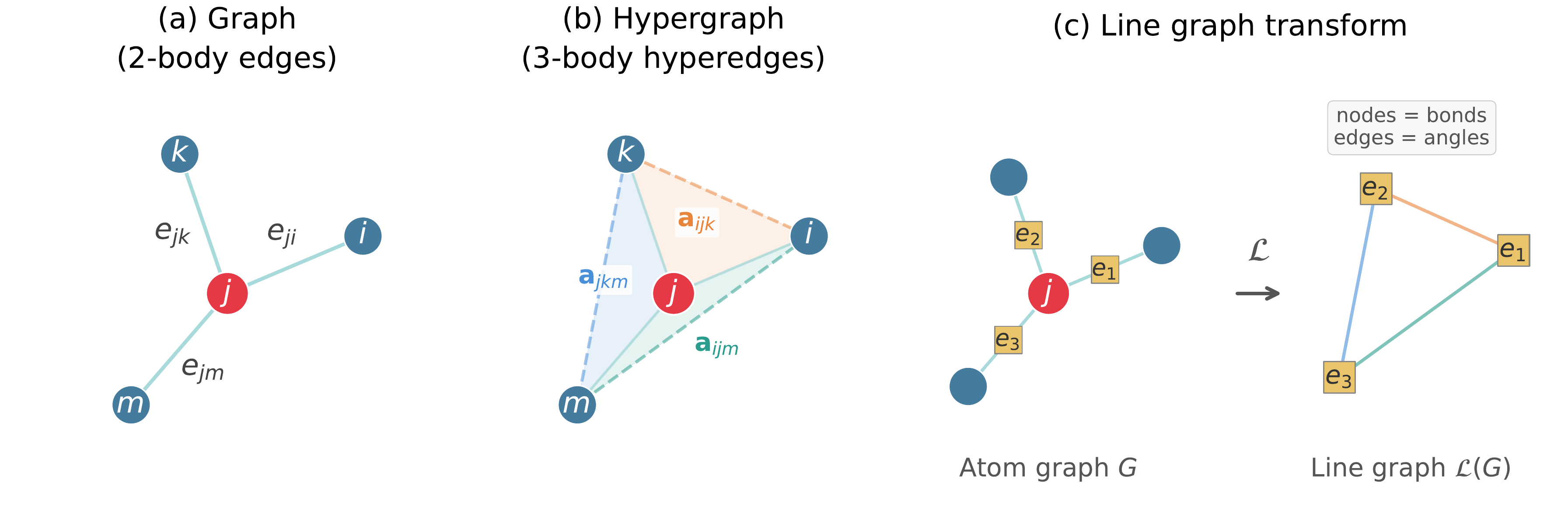}
\caption{Graph structures for atomistic modeling.
  (a)~A graph with 2-body edges, in which each edge~$e_{ji}$
  connects a pair of atoms and carries the interatomic
  distance.
  (b)~A hypergraph with 3-body hyperedges, in which each
  hyperedge~$\bm{a}_{ijk}$ (shaded triangle) connects a
  triplet $(j,i,k)$ and carries the distances and bond
  angle.  The three hyperedges are shown in orange, blue,
  and green.  Dashed lines indicate the outer edge of each
  triangle.
  (c)~The line graph transform~$\calL$, under which each edge in the
  atom graph~$G$ becomes a node (square) in~$\calL(G)$,
  and two nodes in~$\calL(G)$ are connected if the
  corresponding edges share an atom.  The edges in~$\calL(G)$ use the same colors as the
  hyperedges in~(b), showing the correspondence between
  3-body hyperedges and line graph edges.
  ALIGNN, CHGNet, and DPA3 use this two-graph structure to
  implement 3-body message passing.}
\label{fig:architecture}
\end{figure}

\subsection{Graph neural networks for atomistic systems}
\label{sec:gnn-arch}

Fix a center atom~$j$ with neighbor set
$\calN(j) = \{k : k \neq j,\; |\bm{r}_k - \bm{r}_j| < r_c\}$,
where $r_c$ is the per-layer cutoff.
Each atom~$k$ carries a node feature vector
$\bm{n}_k \in \RR^{d}$.

An atomic system is naturally represented as a \emph{graph}, see \cref{fig:architecture}(a), in which
atoms are nodes and pairs within cutoff~$r_c$ are connected by
edges.  A GNN updates node features by
\emph{message passing}.  At each layer, every node~$j$
collects information from its neighbors, aggregates it, and
updates its own feature.  A single message-passing layer has
the general form
\begin{equation}\label{eq:mp}
  \bm{n}_j^{(l+1)} = \psi\Bigl(
    \bm{n}_j^{(l)},\;
    \bigoplus_{k \in \calN(j)}
      s_a(d_{jk})\,
      \phi\bigl(\bm{n}_j^{(l)},\, \bm{n}_k^{(l)},\,
        d_{jk}\bigr)
  \Bigr),
\end{equation}
where $\phi$ is a message function, $\bigoplus$ is a
permutation-invariant aggregation, typically a summation,
$\psi$ is an update function, and $s_a \colon [0, r_c] \to [0, 1]$ is a smooth switch function, positive on $[0, r_c)$ with $s_a(r_c) = s_a'(r_c) = 0$, so that a neighbor's contribution leaves the energy and the forces continuous as it crosses the cutoff boundary.
Both $\phi$ and $\psi$ are
parametrized by MLPs.  The representation must be invariant
under both $O(3)$ and $\Gamma$.  In~\eqref{eq:mp},
$O(3)$-invariance is achieved by using only the scalar
distance~$d_{jk}$ as geometric input, and
$\Gamma$-invariance by the permutation-invariant
aggregation.

In the context of MLIPs, GNNs serve as representations.  The
output feature $\bm{n}_j^{(L)}$ after $L$ layers serves as the
representation, written $\Phi^{(L)}$.  The key advantage of
stacking layers is that the receptive field grows.  After $L$
layers, the feature of atom~$j$ depends on all atoms
reachable by $L$ hops, each shorter than~$r_c$.
This allows the model to
capture interactions at range~$R_c > r_c$ without using a
large single-layer cutoff, which would be computationally
expensive.

However, the framework~\eqref{eq:mp} uses only pairwise
geometric information: the scalar distance~$d_{jk}$
along each edge.
Pozdnyakov et~al.~\cite{pozdnyakov2020incompleteness} showed
that such two-body representations are incomplete, and that
even adding three-body invariant correlations does
not fully resolve the problem, since certain configurations with
identical distance and angle distributions are
geometrically distinct.
\subsection{Hypergraph neural networks}

A \emph{hypergraph}, see~\cref{fig:architecture}(b), generalizes
a graph by allowing \emph{hyperedges} that connect more
than two nodes simultaneously. An HGNN passes
messages along such hyperedges.
In the architecture considered here, each hyperedge contains a center atom~$j$ and two of its neighbors $i$ and~$k$.
It carries the 3-body information, namely the bond-angle cosine~$\cos\theta_{ijk}$ in addition to the distances~$d_{ji}$ and~$d_{jk}$.
The message-passing update aggregates over all hyperedges
incident to a given node, analogous to~\eqref{eq:mp} but with
each message depending on a triplet rather than a pair.  This
gives the HGNN access to angular information that standard
two-body GNNs lack. Unlike the GNN framework of \cref{sec:gnn-arch}, which maintains only node and edge features, the HGNN computes
three-body angle features at each layer.  However, only node
features are carried across layers, and the angle features
are recomputed from scratch at each layer. The HGNN layer comprises an angle-feature calculation and an aggregation operation, defined as follows.
\paragraph{Angle features}
For each ordered pair $i, k \in \calN(j)$ with $i \neq k$:
\begin{equation}\label{eq:angle-feat}
  \bm{a}_{j;\,i,k} = \rho_3\Bigl(
  \bm{n}_j^{(\mathrm{in})},\;
  \bm{n}_i^{(\mathrm{in})},\;
  \bm{n}_k^{(\mathrm{in})},\;
  d_{ji},\;
  d_{jk},\;
    \cos\theta_{ijk}
  \Bigr) \in \RR^{d_a},
\end{equation}
where $\rho_3$ is an MLP and $d_a$ is the angle-feature
dimension.

\paragraph{Aggregation}
\begin{equation}\label{eq:aggregation}
  \bm{n}_j^{(\mathrm{out})} = \psi_3\Bigl(
    \sum_{\substack{i, k \in \calN(j) \\ i \neq k}}
      s_a(d_{ji})\, s_a(d_{jk})\, \bm{a}_{j;\,i,k}
  \Bigr),
\end{equation}
where $\psi_3$ is an MLP and $s_a$ is the switch function of~\eqref{eq:mp}.

\paragraph{Multi-layer stacking}
An $L$-layer HGNN stacks $L$ such layers.  The input node features
for layer~1 are type embeddings: $\bm{n}_k^{(0)} =
\mathrm{type\_embed}(z_k)$.  The output of layer~$l$ provides the
input for layer~$l+1$, in that
$\bm{n}_k^{(l)} \equiv \bm{n}_k^{(\mathrm{out}, l)}$ serves as
$\bm{n}_k^{(\mathrm{in})}$ in \eqref{eq:angle-feat} for the
next layer.  The \emph{$L$-layer HGNN representation}~$\Phi^{(L)}$ maps the $L$-hop surroundings of atom~$j$ to~$\bm{n}_j^{(L)}$.
Its domain is formalized in \cref{def:l-hop}.
The total energy is
$E_\theta = \sum_j f_\theta(\bm{n}_j^{(L)})$, where
$f_\theta$ is a readout MLP.
\subsection{DPA3: GNN on a line graph series}
\label{sec:dpa3-arch}

Three architectures used in practice fall within the scope of our results: ALIGNN~\cite{choudhary2021alignn}, CHGNet~\cite{deng2023chgnet}, and DPA3~\cite{zhang2025dpa3}.
All three maintain node, edge, and angle features and update them in the order angle $\to$ edge $\to$ node.
They differ, however, in the specific form of the message functions and in how angular information is propagated to the nodes.
Among these differences, the last one is particularly consequential.
As we show in~\cref{sec:uat}, it determines which of the three architectures inherits universal approximation.
Specifically, DPA3 and CHGNet build their messages from MLPs, whereas ALIGNN relies on single gated linear layers within its edge-gated convolutions.
We present DPA3 in detail here, as it contains the core ingredients.
CHGNet and ALIGNN are described in \cref{sm:chgnet,sm:alignn} of the supplementary materials, respectively.

DPA3~\cite{zhang2025dpa3} adopts the same angle $\to$ edge
$\to$ node aggregation pattern as ALIGNN and CHGNet,
formalizing it through a \emph{line graph series} (LiGS).
Starting from the atomic graph $G^{(1)} \equiv G$, one
recursively applies the line graph transform to obtain a
series $G^{(1)}, G^{(2)}, G^{(3)}, \ldots$, where
$G^{(k)} = \calL(G^{(k-1)})$.  In particular,
$G^{(2)} = \calL(G)$.  The vertices and edges of
successive graphs correspond to progressively higher-order
geometric entities: $G^{(1)}$ has atoms/bonds, $G^{(2)}$
has bonds/angles, $G^{(3)}$ has angles/dihedrals, and so
forth. The LiGS thus naturally extends to arbitrary order~$K$.  In
practice, DPA3 uses $K = 2$, limiting itself to 3-body angular features, although $K = 3$ with 4-body dihedral features has
also been explored in~\cite{zhang2025dpa3}.

DPA3 uses \emph{directed} edges, so
each bond $\{j,i\}$ gives rise to two directed edges
$(j,i)$ and $(i,j)$, which carry independent features.
DPA3 maintains three feature types: node features
$\bm{n}_j^{(l)} \in \RR^d$, edge features
$\bm{e}_{ji}^{(l)} \in \RR^d$ (which serve as vertex features of the graph~$G^{(2)}$), and angle features
$\bm{a}_{j;i,k}^{(l)} \in \RR^d$ (which serve as edge features
of the graph~$G^{(2)}$).
The initial representations are constructed from type embeddings for nodes, distance embeddings for edges, and angle embeddings for angles.

Each DPA3 layer applies four updates in sequence, each consuming the most recent value of its inputs.\footnote{This ordering is not described in~\cite{zhang2025dpa3}.  It is implemented in the DeePMD-kit source as the \texttt{sequential\_update} option of the DPA3 descriptor, enabled by setting \texttt{sequential\_update = True}.  The default path instead computes all four updates from the layer-$l$ features and combines them additively.}

\smallskip\noindent\emph{1: Edge self-message.}
The edge features first load the node features of both endpoints:
\begin{equation}
  \bm{e}_{ji}^{(l+\frac{1}{2})} = \bm{e}_{ji}^{(l)}
  + \delta_s^{(1)}\, \phi_s^{(1)}\bigl(\bm{n}_j^{(l)},\,
    \bm{n}_i^{(l)},\, \bm{e}_{ji}^{(l)}\bigr).
  \label{eq:dpa3-eself}
\end{equation}

\smallskip\noindent\emph{2: Angle self-message.}
The angle features are updated from the two node-enriched edges that span them:
\begin{equation}
  \bm{a}_{j;i,k}^{(l+1)} = \bm{a}_{j;i,k}^{(l)}
  + \delta_s^{(2)}\, \phi_s^{(2)}\bigl(
    \bm{e}_{ji}^{(l+\frac{1}{2})},\, \bm{e}_{jk}^{(l+\frac{1}{2})},\,
    \bm{a}_{j;i,k}^{(l)}\bigr).
  \label{eq:dpa3-aself}
\end{equation}

\smallskip\noindent\emph{3: Angle $\to$ edge.}
Each edge aggregates over the angles centered at its tail atom:
\begin{equation}
  \bm{e}_{ji}^{(l+1)} = \bm{e}_{ji}^{(l+\frac{1}{2})}
  + \delta_u^{(2)}\, \phi_u^{(2)}\Bigl(
    \sum_{k \in \calN^{(3)}(j) \setminus \{i\}}
    w_{jik}\, \phi_c^{(2)}\bigl(
      \bm{e}_{ji}^{(l+\frac{1}{2})},\, \bm{e}_{jk}^{(l+\frac{1}{2})},\,
      \bm{a}_{j;i,k}^{(l+1)}\bigr)
  \Bigr),
  \label{eq:dpa3-a2e}
\end{equation}
where $\calN^{(3)}(j)$ denotes the angle-cutoff neighbors
of~$j$ ($d_{jk} < r_c^{(3)}$), $w_{jik}$ is a switch
prefactor, and $\delta_u^{(2)}$ is a trainable step size.

\smallskip\noindent\emph{4: Edge $\to$ node.}
Each node aggregates over its edges:
\begin{equation}
  \bm{n}_j^{(l+1)} = \bm{n}_j^{(l)}
  + \delta_u^{(1)}\, \phi_u^{(1)}\Bigl(
    \sum_{i \in \calN^{(2)}(j)}
    w_{ji}\, \phi_c^{(1)}\bigl(
      \bm{n}_j^{(l)},\, \bm{n}_i^{(l)},\,
      \bm{e}_{ji}^{(l+1)}\bigr)
  \Bigr),
  \label{eq:dpa3-e2n}
\end{equation}
where $\calN^{(2)}(j)$ denotes the edge-cutoff neighbors
of~$j$ ($d_{ji} < r_c^{(2)}$).  DPA3 uses two cutoff radii:
$r_c^{(3)} \le r_c^{(2)}$.
All $\phi$ functions ($\phi_s^{(1)}$, $\phi_s^{(2)}$, $\phi_c^{(2)}$,
$\phi_u^{(2)}$, $\phi_c^{(1)}$, $\phi_u^{(1)}$)
are MLPs.
The prefactors are products of switch-function values, $w_{jik} = s_a^{(3)}(d_{ji})\, s_a^{(3)}(d_{jk})$ and $w_{ji} = s_a^{(2)}(d_{ji})$, where $s_a^{(2)}$ and $s_a^{(3)}$ are switch functions as in~\eqref{eq:mp} for the cutoffs $r_c^{(2)}$ and $r_c^{(3)}$.
The $\delta$ step sizes are trainable scalars.
The total energy is
$E_\theta = \sum_j \mathrm{MLP}(\bm{n}_j^{(L)})$.


As noted above, the LiGS extends to arbitrary order~$K$.
Empirically, the DPA3 paper found that $K = 2$
yields optimal performance, while $K = 3$ offers no further improvement and may even degrade accuracy.
Our completeness theory provides an explanation for this observation.
On generic configurations, 3-body information already suffices for universal approximation, so higher body-order features are theoretically unnecessary.

\section{Representation Completeness}
\label{sec:completeness}

\subsection{Completeness and universal approximation}

As described in \cref{sec:hgnn}, most MLIP architectures
decompose the local energy as
$\varepsilon(\calD) = f(\Phi(\calD))$, where
$\Phi : \calE^{(R_c)} \to \RR^d$ is a $G$-invariant representation
and $f : \RR^d \to \RR$ is a readout network.

\begin{definition}[Completeness]
\label{def:completeness}
A \emph{complete} representation at cutoff~$r$ is a continuous $G$-invariant map $\Phi : \calE^{(r)} \to \RR^d$ that separates orbits: for any two local environments $\calD_1, \calD_2 \in \calE^{(r)}$,
\begin{equation}\label{eq:completeness}
  \Phi(\calD_1) = \Phi(\calD_2) \quad\text{implies}\quad
  \calD_1 \sim \calD_2.
\end{equation}
\end{definition}

Since $\sim$ includes both rotations/reflections and
type-preserving permutations, completeness requires
separating orbits under the full symmetry group~$G$.

\begin{definition}[Universal approximation]
\label{def:uat}
A parametrized family $\{\varepsilon_\theta\}$ of continuous functions on $\calE^{(R_c)}$ is a \emph{universal approximator} for $\calF_{R_c}^0$ on a subset $A \subseteq \calE^{(R_c)}$ if for every $\varepsilon \in \calF_{R_c}^0$, every compact $K \subseteq A$, and every $\delta > 0$ there is a parameter~$\theta$ with $\|\varepsilon_\theta - \varepsilon\|_{C^0(K)} < \delta$.
We omit the subset when $A = \calE^{(R_c)}$.
\end{definition}

Completeness constrains the representation alone, whereas universal approximation constrains the architecture $f_\theta \circ \Phi$ as a whole.
That an incomplete representation caps the achievable accuracy is immediate.
The substance of the next result is the converse, that completeness suffices.

\begin{theorem}[Completeness $\Leftrightarrow$ UAT]
\label{thm:completeness-uat}
Let $\Phi : \calE^{(R_c)} \to \RR^d$ be a continuous $G$-invariant representation, and let $\{f_\theta\}$ be a family of MLPs, which by~\cite{cybenko1989approx,hornik1989mlp} approximate every continuous function on~$\RR^d$ uniformly on compact sets.
Then the architecture $\varepsilon_\theta = f_\theta \circ \Phi$ is a universal approximator for $\calF_{R_c}^0$ in the sense of \cref{def:uat} if and only if $\Phi$ is complete.
\end{theorem}

Continuity of~$\Phi$ and compactness of test sets are understood with respect to the Euclidean metric on each stratum.

\begin{proof}
\emph{Completeness $\Rightarrow$ UAT.}
Let $\varepsilon \in \calF_{R_c}^0$ and $K \subset \calE^{(R_c)}$ be compact.
Since $\varepsilon$ is $G$-invariant and $\Phi$ is complete, $\Phi(\calD_1) = \Phi(\calD_2)$ forces $\calD_1 \sim \calD_2$ and hence $\varepsilon(\calD_1) = \varepsilon(\calD_2)$, so $f(\Phi(\calD)) = \varepsilon(\calD)$ defines a function~$f$ on the compact set $\Phi(K)$.
The restriction $\Phi|_K$ is a continuous surjection from a compact space onto a Hausdorff one, hence a quotient map.
The composition $f \circ \Phi|_K = \varepsilon|_K$ is continuous, so $f$ is continuous.
By the Tietze extension theorem, $f$ extends to a continuous function on~$\RR^d$, and by MLP universality, for any $\delta>0$, there exists $f_\theta$ that approximates that extension on $\Phi(K)$ with $\|f_\theta \circ \Phi - \varepsilon\|_{C^0(K)} < \delta$.

\smallskip\noindent\emph{Converse (incompleteness $\Rightarrow$ not UAT).}
If $\Phi$ is incomplete, there exist
$\calD_1 \not\sim \calD_2$ with
$\Phi(\calD_1) = \Phi(\calD_2)$.  For any readout network~$f$,
$f(\Phi(\calD_1)) = f(\Phi(\calD_2))$, so the architecture
cannot distinguish $\calD_1$ and~$\calD_2$.  It remains to
exhibit a target $\varepsilon \in \calF_{R_c}^0$ with
$\varepsilon(\calD_1) \neq \varepsilon(\calD_2)$.

Let $\Phi^{*}$ be a complete representation whose components lie in $\calF_{R_c}^0$.
Such a representation is constructed in \cref{thm:existence} via the switched Gram matrix, without relying on the present theorem, so no circularity arises.
Set $\varepsilon(\calD) = \min\bigl\{1, \|\Phi^{*}(\calD) - \Phi^{*}(\calD_2)\|\bigr\}$.
Then $\varepsilon$ is $G$-invariant and continuous because $\Phi^{*}$ is, and it satisfies the cross-stratum matching of \cref{def:target-space} because the components of~$\Phi^{*}$ do, so $\varepsilon \in \calF_{R_c}^0$.
Moreover $\varepsilon(\calD_2) = 0$, while $\varepsilon(\calD_1) > 0$, because completeness and $\calD_1 \not\sim \calD_2$ give $\Phi^{*}(\calD_1) \neq \Phi^{*}(\calD_2)$.

Taking $K = \{\calD_1, \calD_2\}$ (compact) and
$\delta = \varepsilon(\calD_1)/2$, every model
$f_\theta \circ \Phi$ satisfies
$\|f_\theta \circ \Phi - \varepsilon\|_{C^0(K)}
\ge \varepsilon(\calD_1)/2 > 0$,
so the architecture is not a universal approximator.
\end{proof}
\subsection{Generic configurations and Gram matrix reconstruction}

This subsection is purely information-theoretic.  We show that
complete representations exist, without reference to
any neural network architecture.
Consider a local environment $\calD$ with center type~$z_0$ and
$n$~neighbors at displacements
$\Delta\bm{r}_1, \ldots, \Delta\bm{r}_n \in B_{r_c} \setminus
\{0\}$ with types $z_1, \ldots, z_n$.

\begin{definition}[Generic configuration]
\label{def:generic}
A local environment $\calD \in \calE^{(r)}$ is \emph{generic} if no two same-type neighbors are equidistant from the center, that is, if $z_i = z_k$ and $i \neq k$ imply $|\Delta\bm{r}_i| \neq |\Delta\bm{r}_k|$.
We write $\calE^{*} \subset \calE^{(r)}$ for the set of generic environments.
\end{definition}

On each stratum, the complement of $\calE^{*}$ is the union, over the finitely many same-type pairs $i \neq k$, of the zero sets of the nonzero polynomials $|\Delta\bm{r}_i|^2 - |\Delta\bm{r}_k|^2$.
The zero set of a polynomial that does not vanish identically is closed and has zero Lebesgue measure in $(\RR^3)^n$.
The complement of $\calE^{*}$ is a finite union of such sets, so it is closed and has zero Lebesgue measure.
Hence $\calE^{*}$ is open and dense in $\calE^{(r)}$.

A complete representation has two requirements, invariance under~$G$ and separation of distinct $G$-orbits.
The Gram matrix meets the first for the $O(3)$ factor, and the second amounts to the converse, that it determines the displacement vectors up to a single orthogonal transformation.
That is the content of \cref{lem:gram-recon}.
The permutation factor~$\Gamma$ is handled separately in \cref{lem:gram-orbit}.

\begin{lemma}[Gram matrix reconstruction]
\label{lem:gram-recon}
If\/ $\bm{A}, \bm{B} \in \RR^{n \times 3}$ satisfy
$\bm{A}\bm{A}^T = \bm{B}\bm{B}^T$, then there exists
$R \in O(3)$ such that $\bm{B} = \bm{A}R$.
\end{lemma}

\begin{proof}
It follows from the thin SVD and $\bm{A}\bm{A}^T=\bm{B}\bm{B}^T$ that
\[
\bm{A}=PU\qquad \bm{B}=PV
\]
with $P=(\bm{A}\bm{A}^T)^{1/2}$ and $U^TU=V^TV=I_{3}$.
When $\mathrm{rank}\,\bm{A} < 3$, the factors are not unique.
The column space $\text{col}\,U$ of~$U$ is a 3-dimensional subspace containing $\text{col}\,\bm{A}$, determined by an arbitrary orthonormal completion in the thin SVD, and likewise for~$V$.
Since $\text{col}\,\bm{A} = \text{col}\,P = \text{col}\,\bm{B}$, we choose the two completions so that $\text{col}\,U = \text{col}\,V$.
We claim that
\[
UU^T=VV^T.
\]

Using the fact that for any $x\in\mathbb{R}^3$ and $y\in\mathbb{R}^n$,
\[
((I-UU^T)y,Ux)=(U^Ty,x)-(U^TUU^Ty,x)=(U^Ty,x)-(U^Ty,x)=0.
\]
We conclude that $UU^T$ is the orthogonal projection onto $\text{col}\,U$.
Similarly, $VV^T$ is the orthogonal projection onto $\text{col}\,V$.
The two column spaces coincide by our choice, which gives $UU^T=VV^T$.

Let $R=U^TV$.
It is straightforward to verify that $R\in O(3)$ and $\bm{B}=\bm{A}R$.
\end{proof}
\subsection{Existence of a complete representation}
\label{sec:existence}
\begin{lemma}[Gram matrix orbit characterization]
\label{lem:gram-orbit}
On a stratum~$\calE_{z_0, \bm{z}}$, the Gram matrix map
$\Psi(\calD) = G$ is $O(3)$-invariant and
$\Gamma$-equivariant, where the permutation group $\Gamma = S_{n_1} \times \cdots \times S_{n_T}$ acts on the indices via
$(\pi \cdot G)_{ik} = G_{\pi^{-1}(i),\pi^{-1}(k)}$.

The map~$\Psi$ determines the local
environment up to~$O(3)$-equivalence. More precisely, two environments
$\calD_1, \calD_2$ in the same stratum satisfy
$\calD_1 \sim \calD_2$ if and only if
\begin{equation}\label{eq:gram-equiv}
  \exists\, \pi \in \Gamma:\;
  G^{(2)}_{\pi(i),\pi(k)} = G^{(1)}_{ik} \qquad\forall\, i,k.
\end{equation}
\end{lemma}

\begin{proof}
If $\calD_1 \sim \calD_2$ via
$\Delta\bm{r}_{\pi(i)}^{(2)} = R\Delta\bm{r}_i^{(1)}$, then
$G^{(2)}_{\pi(i),\pi(k)}
= R\Delta\bm{r}_i^{(1)} \cdot R\Delta\bm{r}_k^{(1)}
= G^{(1)}_{ik}$.
Suppose a type-preserving $\pi$ gives
$G^{(2)}_{\pi(i),\pi(k)} = G^{(1)}_{ik}$ for all $i, k$.
Let $\bm{X}^{(1)}$ have rows $\Delta\bm{r}_i^{(1)}$ and $\bm{Y}$
have rows $\Delta\bm{r}_{\pi(i)}^{(2)}$.  Then
$\bm{Y}\bm{Y}^T = \bm{X}^{(1)} (\bm{X}^{(1)})^T$.
By \cref{lem:gram-recon}, there exists $R \in O(3)$ such that $\bm{Y} = \bm{X}^{(1)} R$, hence $\calD_1 \sim \calD_2$.
\end{proof}

\cref{lem:gram-orbit} reduces completeness to separating the $\Gamma$-orbits of Gram matrices.
Since $\Gamma$ is finite, invariant theory supplies finitely many polynomials that do so on each stratum.

Physical applications require continuity with respect to the atomic coordinates, in particular when an atom crosses the boundary at cutoff radius~$R$ and the environment passes from one stratum to another.
A representation built from the Gram matrix alone is discontinuous at such a crossing, because the entries involving a departing neighbor do not decay to zero as $|\Delta\bm{r}_k| \to R$.

To restore continuity, we damp each displacement with a switch function that vanishes smoothly at the cutoff radius.
Let $s\colon [0, R] \to [0, 1]$ be a smooth function, positive on $[0, R)$, satisfying $s(R) = s'(R) = 0$, and such that $r \mapsto s(r)\,r$
is injective on~$(0, R)$. We then define the
\emph{switched Gram matrix} by replacing each displacement
vector by its damping counterpart $\Delta\tilde{\bm{r}}_k =
s(|\Delta\bm{r}_k|)\,\Delta\bm{r}_k$:
\[
  \tilde{G}_{ik}
  = \Delta\tilde{\bm{r}}_i \cdot \Delta\tilde{\bm{r}}_k
  = s(d_i)\,s(d_k)\,G_{ik}.
\]
The switched Gram matrix glues the strata together in the proof of the following proposition.

\begin{proposition}[Existence of a complete representation]
\label{thm:existence}
For any cutoff $R > 0$ there exists a complete representation $\Phi^{*} : \calE^{(R)} \to \RR^{d}$ whose components satisfy the cross-stratum matching of condition~3 in \cref{def:target-space}, with $R_c$ replaced by~$R$.
\end{proposition}

\begin{proof}
\emph{Step~1: a single stratum.}
Work on a fixed stratum~$\calE_{z_0, \bm{z}}^{(R)}$ with
$n$~neighbors.  The Gram matrix
$\Psi(\calD) = G$ takes values in the space of
$n \times n$ real symmetric matrices
$\calV \cong \RR^{n(n+1)/2}$.  By \cref{lem:gram-orbit},
$\calD_1 \sim \calD_2$ if and only if $\Psi(\calD_1)$ and
$\Psi(\calD_2)$ lie in the same $\Gamma$-orbit under the
action $G \mapsto P_\pi G P_\pi^T$.

Since $\Gamma$ is a finite group acting on~$\calV$,
by Noether's theorem~\cite{noether1915koerper} the ring of
$\Gamma$-invariant polynomials $\RR[\calV]^\Gamma$ is finitely
generated, so there exist $\Gamma$-invariant polynomials
$p_1, \ldots, p_d : \calV \to \RR$ such that every
$\Gamma$-invariant polynomial is a polynomial in
$p_1, \ldots, p_d$.  Moreover, these generators separate
$\Gamma$-orbits, in the sense that if $p_i(G^{(1)}) = p_i(G^{(2)})$ for all
$i = 1, \ldots, d$, then $G^{(1)}$ and $G^{(2)}$ lie in the
same $\Gamma$-orbit.
Indeed, for any
$G^{(2)} \notin \Gamma \cdot G^{(1)}$, the polynomial
$P(X) = \prod_{M \in \Gamma \cdot G^{(1)}} \|X - M\|^2$
is $\Gamma$-invariant, vanishes on $\Gamma \cdot G^{(1)}$,
and is strictly positive on $\Gamma \cdot G^{(2)}$.
Since $P$ is $\Gamma$-invariant, it is expressible in the
generators, so at least one generator differs between the
two orbits.

Define
\[
  \Phi(\calD) = \bigl(p_1(\Psi(\calD)),\; \ldots,\;
    p_d(\Psi(\calD))\bigr) \in \RR^d.
\]
$\Phi$ is continuous because its components are polynomials in the entries of the continuous map~$\Psi$, $O(3)$-invariant because $\Psi$ is, and $\Gamma$-invariant because each $p_i$ is, hence $G$-invariant.  If
$\Phi(\calD_1) = \Phi(\calD_2)$, then all generators agree on
$\Psi(\calD_1)$ and $\Psi(\calD_2)$, so
$\Psi(\calD_1) \sim_\Gamma \Psi(\calD_2)$, hence
$\calD_1 \sim \calD_2$ by \cref{lem:gram-orbit}.

\emph{Step~2: assembly across strata.}
The per-stratum maps are now glued through the switched Gram matrix.
Pad $\tilde{G}$ by zeros to a fixed size, reserving for each type~$t$ as many index slots as the maximal number~$M_t$ of type-$t$ neighbors, and write $\hat{G}(\calD)$ for the result.
When a neighbor leaves through the boundary, its switched entries vanish, so $\hat{G}(\calD) \to \hat{G}(\calD^-)$ along the transitions of condition~3.
The pair $(z_0, \hat{G})$ still determines the environment.
The positive diagonal entries reveal the number of neighbors of each type, hence the stratum, and the injectivity of $r \mapsto s(r)\,r$ recovers the distances, then the values $s(d_i) > 0$, and finally $G_{ik} = \hat{G}_{ik}/(s(d_i)\,s(d_k))$, so \cref{lem:gram-recon,lem:gram-orbit} apply.
Repeating Step~1 for the finite group $\bar{\Gamma} = \prod_t S_{M_t}$, which permutes the slots within each type block of~$\hat{G}$, and appending the center type~$z_0$, yields~$\Phi^{*}$.
Its components are polynomials in the entries of~$\hat{G}$, so they inherit the continuity and the boundary matching, and the center type is constant along every condition-3 transition.
\end{proof}

$\Phi^{*}$ is the representation invoked in the converse direction of \cref{thm:completeness-uat}.
\section{HGNN Approximation of Complete Representations}
\label{sec:hgnn-approx}
We have thus established the existence of complete representations, yet the construction of \cref{thm:existence} is information-theoretic.
It assembles $\Phi^{*}$ from invariant polynomials rather than from a trainable network.
We now turn to the question of whether such a representation can be realized in practice.
Specifically, we show that the HGNN architecture is expressive enough to approximate any such complete representation to arbitrary accuracy on compact sets.
\subsection{DeepSets universality}

We first record a form of the DeepSets universality theorem~\cite{zaheer2017deepsets,wagstaff2019limitations} that covers vector-valued elements.
It is the tool underlying every approximation result below.
The HGNN aggregation~\eqref{eq:aggregation} is a sum of per-triplet messages, and DeepSets universality is what lets such a sum represent an arbitrary symmetric function of the triplet multiset.

For a compact $X \subset \RR^d$, let $\Mset{M}{X} \equiv \binom{X}{\le M}$\footnote{The notation $\binom{X}{\le M}$ follows the convention in which $\binom{X}{m}$ denotes the collection of $m$-element sub-multisets of~$X$; the subscript $\le M$ admits every size up to~$M$.} denote the set of all multisets of elements of~$X$ having size at most~$M$.
Each fixed-size stratum
$\Mset{m}{X} = X^m / S_m$ carries the quotient metric
$d(S, S') =
\min_{\pi \in S_m} \max_i \|x_i - y_{\pi(i)}\|$, where $S = \{x_1, \ldots, x_m\}$ and
$S' = \{y_1, \ldots, y_m\}$ are any ordered representatives,
and $\Mset{M}{X} = \bigsqcup_{m=0}^{M} \Mset{m}{X}$
has the disjoint-union topology.
Since $X$ is compact, $X^m$ is compact as a finite product of compact sets, and $\Mset{m}{X}$ is compact as the continuous image of~$X^m$ under the quotient map.

\begin{lemma}[DeepSets universality]
\label{lem:deepsets}
Let $X \subset \RR^d$ be compact.
There is an integer $m$ depending only on $M$ and~$d$ such that, for any continuous function
$f\colon \Mset{M}{X} \to \RR$, there exist continuous
functions
$\phi\colon X \to \RR^{m}$ and $\rho\colon \RR^{m} \to \RR$
such that
$f(S) = \rho\bigl(\sum_{x \in S} \phi(x)\bigr)$
for every $S \in \Mset{M}{X}$, where $\sum_{x \in S}$
denotes the sum over elements of~$S$ with multiplicity.
One may take $\phi(x) = (x^{\alpha})_{0 \le |\alpha| \le M}$, the monomials of degree at most~$M$, so that $m = \binom{M+d}{d}$.
\end{lemma}

\begin{proof}
The coordinate $\alpha = 0$ recovers the size of~$S$, and for each fixed size the sums $\sum_{x \in S} x^{\alpha}$ with $1 \le |\alpha| \le M$ generate the polynomial invariants of the symmetric group acting on tuples of vectors~\cite{weyl1946classical}, which separate its orbits as in Step~1 of the proof of \cref{thm:existence}.
Hence $S \mapsto \sum_{x \in S} \phi(x)$ is injective, and, being a continuous bijection from a compact space onto its image, a homeomorphism.
Composing $f$ with the inverse and extending by the Tietze theorem yields~$\rho$.
\end{proof}

\cref{lem:deepsets} produces continuous $\phi$ and~$\rho$, whereas the HGNN's message functions are MLPs.
The following lemma may be viewed as an approximate version of DeepSets universality.
\begin{lemma}[DeepSets MLP approximation]
\label{lem:deepsets-mlp}
Let $X \subset \RR^d$ be compact,
$K \subset \Mset{M}{X}$ compact, and
$f \in C^0(K)$ a continuous symmetric function.
For every $\delta > 0$, there exist MLPs
$\phi_\theta \colon X \to \RR^{m}$ and
$\rho_\theta \colon \RR^{m} \to \RR$, with $m$ as in \cref{lem:deepsets}, such that
\[
  \bigl\|f - \rho_\theta \circ
  \textstyle\bigl(\sum_i \phi_\theta\bigr)
  \bigr\|_{C^0(K)} < \delta.
\]
\end{lemma}

\begin{proof}
For a multiset $S \in \Mset{M}{X}$, we write
$\mathrm{elts}(S) \subset X$ for the set of elements
appearing in~$S$ (forgetting multiplicities) and
$\sum_{x \in S}$ for the sum over elements with
multiplicity.

By \cref{lem:deepsets},
$f(S) = \rho\bigl(\sum_{x \in S} \phi(x)\bigr)$
for continuous $\phi \colon X \to \RR^{m}$ and
$\rho \colon \RR^{m} \to \RR$.
Define $X_K = \bigcup_{S \in K} \mathrm{elts}(S)
\subset X$, which is compact because $K$ is compact, and
$\Sigma = \bigl\{\sum_{x \in S} \phi(x) : S \in K\bigr\}
\subset \RR^{m}$, which is also compact because it is a continuous image of~$K$.

By MLP universality on~$X_K$, there exists
$\phi_\theta$ with
$\|\phi_\theta - \phi\|_{C^0(X_K)} < \epsilon$.
Then for every $S \in K$,
$\bigl\|\sum_{x \in S} \phi_\theta(x) -
\sum_{x \in S} \phi(x)\bigr\|
\le M \epsilon$.
Let $\Sigma_\epsilon$ be the closed $M\epsilon$-neighborhood
of~$\Sigma$, which is compact.
By MLP universality on~$\Sigma_\epsilon$, there exists
$\rho_\theta$ with
$\|\rho_\theta - \rho\|_{C^0(\Sigma_\epsilon)} < \epsilon$.
Since $\rho$ is continuous on the compact
set~$\Sigma_\epsilon$, it is uniformly continuous: for any
$\delta' > 0$ there exists $\eta > 0$ such that
$\|s - s'\| < \eta \Rightarrow |\rho(s) - \rho(s')| <
\delta'$ for all $s, s' \in \Sigma_\epsilon$.
Set $\delta' = \delta / 2$, let $\eta$ be the corresponding modulus, and choose $\epsilon$ small enough that $M\epsilon < \eta$ and $\epsilon < \delta / 2$.
For every $S \in K$, writing
$s = \sum_{x \in S} \phi(x) \in \Sigma$ and
$\tilde{s} = \sum_{x \in S} \phi_\theta(x) \in
\Sigma_\epsilon$,
\[
  |\rho_\theta(\tilde{s}) - \rho(s)|
  \le |\rho_\theta(\tilde{s}) - \rho(\tilde{s})|
  + |\rho(\tilde{s}) - \rho(s)|
  < \tfrac{\delta}{2} + \tfrac{\delta}{2} = \delta,
\]
where the first term is bounded by the choice of~$\rho_\theta$ and the second by uniform continuity, since $\|\tilde{s} - s\| \le M\epsilon < \eta$.
Since the bound holds for every $S \in K$,
$\|f - \rho_\theta \circ (\sum \phi_\theta)\|_{C^0(K)}
< \delta$.
\end{proof}

\subsection{Single-layer approximation}

\cref{lem:deepsets-mlp} shows that a pair of MLPs can approximate any symmetric function of a multiset.
To apply this result to a single HGNN layer, we require that the multiset over which the layer aggregates encodes the same information as the Gram matrix. The following identities demonstrate that this is indeed the case.
The map from the Gram matrix~$G$ to distances and angles,
\begin{equation}\label{eq:gram-to-dist}
  d_i = \sqrt{G_{ii}}, \qquad
  \cos\theta_{ijk} = \frac{G_{ik}}{d_i\, d_k}
  \quad (i \neq k),
\end{equation}
is a bijection with inverse
\begin{equation}\label{eq:dist-to-gram}
  G_{ik} = d_i\, d_k\, \cos\theta_{ijk}, \qquad
  G_{ii} = d_i^2.
\end{equation}

The two descriptions of the local geometry are thus interchangeable.
The computation performed by a single HGNN layer can therefore be equivalently formulated in terms of the Gram matrix.
The following lemma establishes that, under the condition that distinct neighbors carry distinct labels, this layer can approximate any continuous $\Gamma$-invariant function of that Gram matrix and the incoming node features.
\begin{lemma}[Architectural expressiveness]
\label{lem:arch-express}
On a stratum~$\calE_{z_0, \bm{z}}^{(r_c)}$ with $n \ge 2$ neighbors and permutation group~$\Gamma$, consider a single HGNN layer (\eqref{eq:angle-feat}--\eqref{eq:aggregation}).
Suppose the pairs $(\bm{n}_k^{(\mathrm{in})}, d_{jk})$, $k \in \calN(j)$, are pairwise distinct.
Suppose further that the node features of neighbors of different types take values in disjoint sets.
Then, with sufficient MLP capacity and sufficiently large angle-feature dimension~$d_a$, the HGNN layer can approximate any continuous $\Gamma$-invariant function of the labeled data $\calL = (G^{(j)}, (\bm{n}_1^{(\mathrm{in})}, \ldots, \bm{n}_n^{(\mathrm{in})}))$, where $G^{(j)}$ is the Gram matrix~\eqref{eq:gram} of $j$'s environment and $\Gamma$ acts by $\pi \cdot (G, \bm{n}) = (P_\pi G P_\pi^T, \bm{n}_{\pi^{-1}})$.
\end{lemma}

The argument proceeds in three steps. First, a single HGNN layer has DeepSets form over the multiset of neighbor triplets, so \cref{lem:deepsets-mlp} applies. Second, every continuous $\Gamma$-invariant function of the labeled data factors through that multiset. Third, the two combine by composition.

Before delving into the proof, we first fix notation.
Let $\calV = \RR^{n(n+1)/2}_{\mathrm{sym}}$ denote the space
of Gram matrices and
$\calW = \calV \times (\RR^{d})^n$ the space of labeled
data $\calL = (G, (\bm{n}_1, \ldots, \bm{n}_n))$.
The group~$\Gamma$ acts on~$\calW$ by
$\pi \cdot (G, \bm{n}) = (P_\pi G P_\pi^T,
\bm{n}_{\pi^{-1}})$. Let $X = \RR^d \times \RR^d \times (0, r_c) \times (0, r_c)\times [-1, 1]$ be the space of a single tuple
$(\bm{n}_i, \bm{n}_k, d_{ji}, d_{jk}, \cos\theta_{ijk})$.
Define the map $\mu \colon \calW \to
\Mset{M(M-1)}{X}$ by
\begin{equation}\label{eq:calS}
\mu(G, \bm{n}) = \bigl\{(\bm{n}_i^{(\mathrm{in})},\,
\bm{n}_k^{(\mathrm{in})},\, d_{ji},\, d_{jk},\,
\cos\theta_{ijk})\bigr\}_{i \neq k},
\end{equation}
where distances and angles are obtained from~$G$ via
~\eqref{eq:gram-to-dist} and the index pair $(i,k)$
runs over all ordered pairs of distinct neighbors.
Write $C^0_{\mathrm{sym}}(\Mset{M(M-1)}{X})$ for the
continuous functions on the multiset space,
$C_\Gamma(\calW)$ for the continuous $\Gamma$-invariant
functions on~$\calW$, and $\mu^{*} \colon g \mapsto g \circ \mu$ for the pullback along~$\mu$.

\begin{proof}
\smallskip\noindent\emph{Step~1: DeepSets universality.}
The HGNN computes
\[
  \bm{n}_j^{(\mathrm{out})} = \psi_3\Bigl(
    \sum_{\substack{i, k \in \calN(j) \\ i \neq k}}
      \tilde{\rho}_3\bigl(
        \bm{n}_j^{(\mathrm{in})},\,
        \bm{n}_i^{(\mathrm{in})},\,
        \bm{n}_k^{(\mathrm{in})},\,
        d_{ji},\, d_{jk},\,
        \cos\theta_{ijk}\bigr)
  \Bigr),
\]
where $\tilde{\rho}_3$ absorbs the switch factors
$s_a(d_{ji})\,s_a(d_{jk})$ into~$\rho_3$, both being continuous functions of the inputs $d_{ji}$ and~$d_{jk}$ of~$\rho_3$.
Conversely, since $s_a$ is positive on $[0, r_c)$, on compact sets the desired summand is realized by letting $\rho_3$ approximate the quotient by the switch factors.
This has DeepSets form
$\psi_3 \circ \sum \circ \tilde{\rho}_3$
over~$\mu(\calL)$, where each summand $\tilde{\rho}_3$ acts on one
element of the multiset.
The multiset has at most $M(M-1)$ elements.
For any $g \in C^0_{\mathrm{sym}}(\Mset{M(M-1)}{X})$, compact
$K' \subset \Mset{M(M-1)}{X}$, and $\delta > 0$, \cref{lem:deepsets-mlp}, applied with sufficiently large~$d_a$ and with~$X$ replaced by a compact subset containing the tuple values over~$K'$, yields MLP weights for $\rho_3$ and~$\psi_3$ such that
$\|\bm{n}_j^{(\mathrm{out})} - g\|_{C^0(K')} < \delta$.

\smallskip\noindent\emph{Step~2: $\mu^*$ maps
$C^0_{\mathrm{sym}}(\Mset{M(M-1)}{X})$ into $C_\Gamma(\calW)$ and is
surjective.}

\emph{$\mu^*$ is well-defined.}
The map~$\mu$ is continuous and $\Gamma$-invariant.
Continuity holds because each element of the multiset depends continuously on~$\calL$ via~\eqref{eq:gram-to-dist}.
For invariance, any $\pi \in \Gamma \subset S_n$ permutes the ordered pairs $(i,k)$ with $i \neq k$ among themselves, and the multiset is unordered, so $\mu(\pi \cdot \calL) = \mu(\calL)$.
Hence for $g \in C^0_{\mathrm{sym}}(\Mset{M(M-1)}{X})$, the pullback
$g \circ \mu$ is continuous (composition of continuous maps)
and $\Gamma$-invariant, i.e.,
$g \circ \mu \in C_\Gamma(\calW)$.

\emph{$\mu^*$ is surjective.}
It suffices to show that $\mu$ is injective modulo~$\Gamma$,
i.e., $\mu(\calL_1) = \mu(\calL_2) \Rightarrow
\calL_1 \sim_\Gamma \calL_2$.
Since $(\bm{n}_k^{(\mathrm{in})}, d_{jk})$ is distinct for
distinct~$k$ (by hypothesis), each element
$(\bm{n}_i, \bm{n}_k, d_{ji}, d_{jk}, \cos\theta_{ijk})$
of the multiset uniquely identifies the atoms $i$ and~$k$.
Equal multisets therefore imply a permutation $\pi \in S_n$
with $\bm{n}_{\pi(k)}^{(2)}
= \bm{n}_k^{(1)}$,
$d_{j,\pi(k)}^{(2)} = d_{jk}^{(1)}$, and
$\cos\theta_{\pi(i)j\pi(k)}^{(2)} = \cos\theta_{ijk}^{(1)}$ for all
$i \neq k$, where the superscripts refer to $\calL_1$ and~$\calL_2$.
Since the node features of neighbors of different types lie in disjoint sets, the matching $\bm{n}_{\pi(k)}^{(2)} = \bm{n}_k^{(1)}$ forces~$\pi$ to be
type-preserving, that is, $\pi \in \Gamma$.
By \eqref{eq:dist-to-gram}, the matched distances and angles give
$G^{(2)}_{\pi(i),\pi(k)} = G^{(1)}_{ik}$, so
$\calL_1 \sim_{\Gamma} \calL_2$.

Now let $f \in C_\Gamma(\calW)$ and
$K \subset \calW$ compact.
Since~$\mu$ is injective modulo~$\Gamma$,
$\mu(\calL_1) = \mu(\calL_2) \Rightarrow
f(\calL_1) = f(\calL_2)$,
so $f$ factors as $f = g_0 \circ \mu$ for a well-defined
function~$g_0$ on the image~$\mu(K)$.
The function~$g_0$ is continuous on~$\mu(K)$.
The restriction $\mu|_K$ is a continuous map from the compact
set~$K$ to the metrizable space
$\Mset{M(M-1)}{X}$, hence a closed map, and for any closed $C \subset \RR$,
$g_0^{-1}(C) = \mu(f^{-1}(C) \cap K)$ is closed
in~$\mu(K)$.
Since $\mu(K)$ is a compact subset of
the metrizable space $\Mset{M(M-1)}{X}$,
the Tietze extension theorem provides a continuous extension
$g \colon \Mset{M(M-1)}{X} \to \RR$ with
$g|_{\mu(K)} = g_0$.

\smallskip\noindent\emph{Step~3: approximation.}
Let $f \in C_\Gamma(\calW)$,
$K \subset \calW$ compact, and $\delta > 0$.
By Step~2, $f = g \circ \mu$ on~$K$ with
$g \in C^0_{\mathrm{sym}}(\Mset{M(M-1)}{X})$.
By Step~1, there exist MLP weights such that the HGNN
output~$\tilde{g}$ 
satisfies
$\|\tilde{g} - g\|_{C^0(\mu(K))} < \delta$.
For every $\calL \in K$,
\[
  |(\tilde{g} \circ \mu)(\calL) - f(\calL)|
  = |\tilde{g}(\mu(\calL)) - g(\mu(\calL))| < \delta,
\]
so the HGNN approximates~$f$ uniformly on~$K$.
\end{proof}

Applying this procedure to each component of the complete representation of~\cref{thm:existence}, we prove that a single HGNN layer approximates the representation on generic environments.
\begin{proposition}[Single-layer HGNN approximates
  a complete representation]
\label{thm:single-hgnn}
Let the HGNN have a single layer with
$r_c^{(3)} = r_c^{(2)} = r_c$ and sufficient feature dimensions $d_a$ and~$d$.
For any compact $K \subset \calE^{(r_c)} \cap \calE^*$ whose environments have at least two neighbors, and any $\delta > 0$,
there exist MLP weights such that the HGNN output approximates the complete representation $\Phi^{*}$ of \cref{thm:existence} with
$\|\Phi - \Phi^*\|_{C^0(K)} < \delta$.
\end{proposition}

\begin{proof}
Work on a fixed stratum
$\calE_{z_0, \bm{z}}^{(r_c)}$.  The cross-stratum
case follows because type embeddings place each stratum's
data in a disjoint region of the MLP input space, and by MLP
universality a single set of weights can approximate
independent targets on each of the finitely many strata.

By \cref{thm:existence}, the components of the complete representation~$\Phi^{*}$, restricted to the stratum, are polynomials in the entries of the switched Gram matrix.
These entries are continuous functions of the distances and angles, so each component is a continuous $\Gamma$-invariant function of the labeled data.

For the single-layer case, the input node features are type
embeddings
$\bm{n}_k^{(0)} = \mathrm{type\_embed}(z_k)$.  For generic
configurations, distinct same-type atoms have distinct
distances~$d_{jk}$, so the pair
$(\bm{n}_k^{(0)}, d_{jk})$ is distinct for distinct~$k$.
The type embeddings are chosen distinct, so the features of neighbors of different types lie in disjoint sets, and both hypotheses of \cref{lem:arch-express} are satisfied.

By \cref{lem:arch-express} applied to each component,
there exist MLP weights such that the HGNN output
$\Phi(\calD)$ approximates $\Phi^*(\calD)$ to
within~$\delta$ uniformly on~$K$.
\end{proof}

\subsection{Multi-layer approximation}

We now extend the single-layer result to the multi-layer HGNN.
We first introduce the $L$-hop neighborhoods and environments that an $L$-layer network sees, then the two conditions on the configuration that the main theorem requires.
These are connectivity, which makes $L$ layers reach the physical range~$R_c$, and overlap, which lets each layer reconstruct the extended geometry from local data.
We then show that this reconstruction is possible in \cref{lem:info-suff} and prove the multi-layer approximation theorem in \cref{thm:multi-hgnn}. Two collections of Gram matrices recur in what follows.
For a center atom~$j$,
\begin{equation}\label{eq:info-data}
  \calI_j = \bigl(G^{(j)},\;
  \{G^{(k)}\}_{k \in \calN(j)}\bigr)
\end{equation}
gathers the Gram matrices of $j$ and of each of its neighbors, while $\bar{G}^{(j)}$ denotes the Gram matrix of the extended environment obtained by merging all the neighbors' environments.

\begin{definition}[$L$-hop neighborhood and environment]
\label{def:l-hop}
For cutoff~$r_c$, define the 1-hop neighbor set
$\calN(j) = \{k : k \neq j,\; |\bm{r}_k - \bm{r}_j| <
r_c\}$.
The $L$-hop neighbor set is defined inductively by $\calN^1(j) = \calN(j)$ and $\calN^{l+1}(j) = \bigl(\calN(j) \cup \bigcup_{k \in \calN(j)} \calN^l(k)\bigr) \setminus \{j\}$, so that the center never counts among its own neighbors.
The $L$-hop local environment is
\begin{equation}\label{eq:l-hop-env}
  \calD_j^{[r_c, L]} = \bigl(z_j,\; \{(\Delta\bm{r}_{jk},\, z_k)
    : k \in \calN^L(j)\}\bigr),
  \qquad \Delta\bm{r}_{jk} = \bm{r}_k - \bm{r}_j.
\end{equation}
\end{definition}

The $L$-hop neighborhood $\calN^L(j)$ contains every atom
reachable from~$j$ by a chain of at most $L$ steps, each of
length less than~$r_c$.  Note that $\calN^L(j)$ is in general
a \emph{subset} of the ball $\{k : |\bm{r}_k - \bm{r}_j| < Lr_c\}$.
An atom within distance~$Lr_c$ may not be reachable
if there are gaps in the neighbor structure.

\begin{definition}[$L$-hop environment space]
\label{def:l-hop-space}
The \emph{$L$-hop environment space} $\calE^{[r_c, L]}$ is the set of all $L$-hop local environments $\calD_j^{[r_c, L]}$, as $j$ ranges over the atoms of all configurations.
\end{definition}

Since every atom reachable from~$j$ in at most $L$ hops of length less than $r_c$ lies within distance~$Lr_c$ of~$j$, we have the inclusion $\calE^{[r_c, L]} \subset \calE^{(Lr_c)}$.
The membership condition is intrinsic.
An environment $\calD \in \calE^{(Lr_c)}$ lies in $\calE^{[r_c, L]}$ if and only if every neighbor of the central atom is connected to it by a chain of at most $L$ steps, with each step linking two atoms at distance less than $r_c$.
This condition can be expressed as a finite disjunction over all such chains of finitely many strict inequalities.
Hence $\calE^{[r_c, L]}$ is an open subset of $\calE^{(Lr_c)}$, inheriting its stratification, metric, and differentiable structure.
We call an $L$-hop environment \emph{generic} if the local environment of every atom in it is generic in the sense of \cref{def:generic} and no two of its atoms lie at distance exactly~$r_c$.
\begin{definition}[Connectivity condition]
\label{def:connectivity}
A configuration satisfies the \emph{connectivity condition}
with respect to cutoff~$r_c$ and physical range~$R_c$ if,
for every atom~$j$ and every atom~$k \neq j$ with
$|\bm{r}_k - \bm{r}_j| < R_c$, there exists $L \ge 1$ such
that $k \in \calN^L(j)$.
\end{definition}

Since a configuration has finitely many atoms, the hop depths required for the individual pairs admit a finite maximum. Hence, a single~$L$ suffices for all pairs simultaneously. Consequently, $\calD_j^{[r_c, L]} \supseteq \calD_j^{(R_c)}$ for every atom~$j$, meaning that an 
$L$-layer HGNN effectively covers the entire physical interaction range 
$R_c$. See \cref{fig:receptive-field} for an illustration.
In condensed matter systems, where atoms typically form a connected network at cutoff~$r_c$, this condition is generally satisfied.
\begin{figure}[t]
\centering
\includegraphics[width=\textwidth]{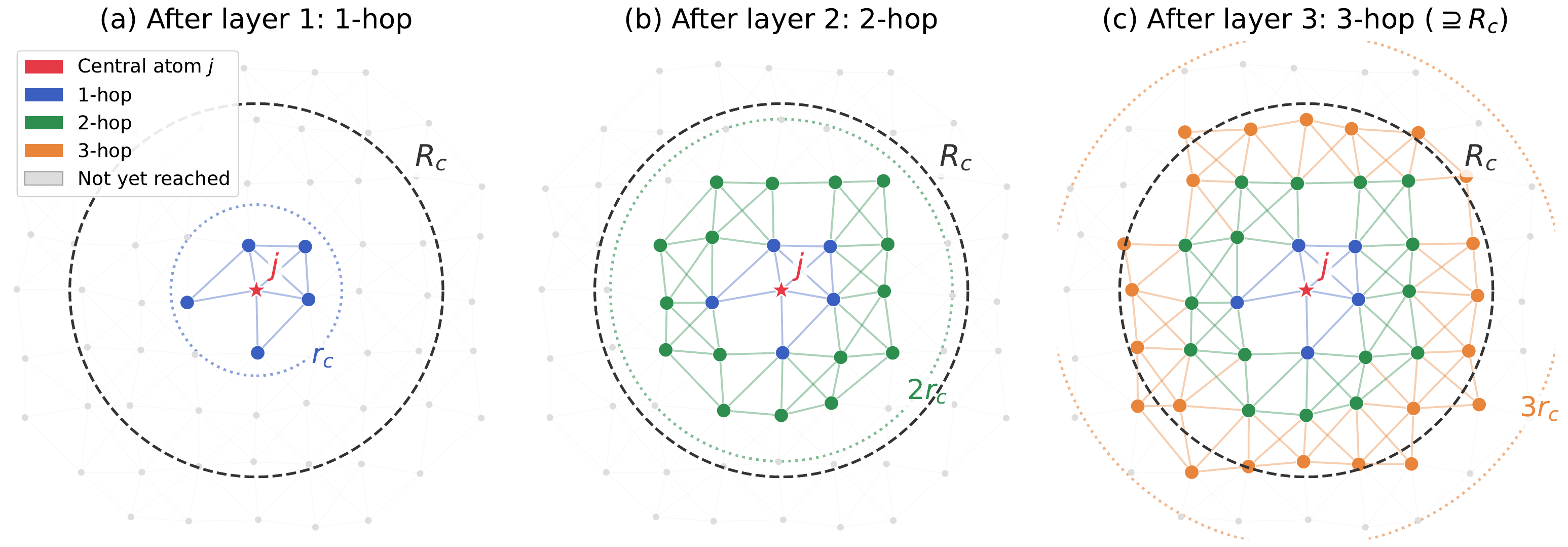}
\caption{Extending the receptive field via multi-layer message
  passing.  The central atom~$j$ (star) has a per-layer
  cutoff~$r_c$.  The dashed circle marks the physical
  interaction range~$R_c$.
  (a)~After one layer, only the 1-hop neighbors (blue)
  are reached.
  (b)~After two layers, the 2-hop neighborhood (green)
  extends further, but atoms near the boundary of~$R_c$
  remain unreached.
  (c)~After three layers, the 3-hop neighborhood (orange)
  covers all atoms within~$R_c$.
  Note that the $L$-hop neighborhood is in general a
  \emph{subset} of the ball of radius~$Lr_c$ (dotted
  circles), since an atom within distance~$Lr_c$ may not be
  reachable in $L$ graph hops if intermediate neighbors
  are absent.  The connectivity condition
  guarantees that sufficiently
  many layers cover~$R_c$.  The overlap condition
  ensures that each layer faithfully
  reconstructs the extended geometry from local data.}
\label{fig:receptive-field}
\end{figure}

\begin{definition}[Overlap condition]
\label{def:overlap}
Write $\calN[j] = \calN(j) \cup \{j\}$ for the closed neighborhood of atom~$j$.
A configuration satisfies the \emph{overlap condition} with respect to~$r_c$ if every atom has a neighbor and, for every ordered pair of atoms $(j, k)$ with $|\bm{r}_j - \bm{r}_k| < r_c$, the displacements $\{\bm{r}_m - \bm{r}_k : m \in \calN[j] \cap \calN[k]\}$ span~$\RR^3$.
\end{definition}

Both $j$ and $k$ lie in $\calN[j] \cap \calN[k]$, since $|\bm{r}_j - \bm{r}_k| < r_c$.
Spanning $\RR^3$ is equivalent to the existence of three atoms $m_1, m_2, m_3$ in $\calN[j] \cap \calN[k]$ whose displacements relative to~$k$ are linearly independent, which is the form used in \cref{lem:info-suff}.
Three is exactly the number required there, since the relative orientation of the local frames of $j$ and~$k$ is recovered as a linear map of~$\RR^3$, and such a map is determined by its values on a basis.

The overlap condition requires $r_c$ to be large enough that two neighboring atoms share enough common neighbors to span~$\RR^3$.
In condensed matter systems, a cutoff slightly above the nearest-neighbor distance typically suffices.

\begin{definition}[Overlap condition for $L$-hop environments]
\label{def:overlap-env}
The \emph{hop depth} of an atom in an $L$-hop environment is the smallest number of hops connecting it to the center.
An $L$-hop environment satisfies the \emph{overlap condition} if its center has a neighbor and the spanning requirement of \cref{def:overlap} holds for every ordered pair of its atoms $(j, k)$ with $|\bm{r}_j - \bm{r}_k| < r_c$ in which $j$ has hop depth at most $L - 1$.
Here the neighborhoods are computed within the environment.
\end{definition}

The depth restriction exempts the outermost atoms, whose shared neighbors may lie beyond the $L$-hop horizon.
For a pair whose first atom has hop depth at most $L - 1$, the shared neighbors lie within one hop of that atom, hence at depth at most~$L$.
Consequently, every $L$-hop environment of a configuration satisfying the overlap condition satisfies \cref{def:overlap-env}.

The inductive step of the multi-layer theorem needs a center atom to recover its extended environment from what its neighbors already encode.
At the level of Gram matrices, we have
%
\begin{lemma}[Information sufficiency]
\label{lem:info-suff}
Under the overlap condition at~$r_c$, the data~$\calI_j$ uniquely determines the
extended Gram matrix $\bar{G}^{(j)}$.  Moreover,
$\bar{G}^{(j)}$ is a continuous function of~$\calI_j$.
\end{lemma}

\begin{proof}
Every atom $a$ in the extended environment belongs to the
environment of at least one $k_a \in \calN(j)$.  We compute
each entry $\bar{G}^{(j)}_{ab}$ for two cases.

\smallskip\noindent\emph{Case~1: $a, b$ belong to a common
neighbor's environment.}
Let $k \in \calN(j)$ have both $a$ and~$b$ in its
environment.  Since $k \in \calN(j)$ implies
$j \in \calN(k)$, atom~$j$ appears in $G^{(k)}$.  Direct
computation gives
\begin{equation}\label{eq:gram-case1}
  \bar{G}^{(j)}_{ab}
  = G^{(k)}_{ab} - G^{(k)}_{aj} - G^{(k)}_{bj} + d_{jk}^2,
\end{equation}
where $d_{jk}^2 = G^{(j)}_{kk}$.  All quantities on the
right-hand side are known from~$\calI_j$.

\smallskip\noindent\emph{Case~2: $a$ and $b$ belong to different
neighbors' environments with no common environment containing
both.}
Let $a$ belong to $k_a$'s environment and $b$ to $k_b$'s, with
$k_a \neq k_b$.  This requires determining the relative
orientation between $k_a$'s and $k_b$'s local environments via
frame alignment.

\smallskip\noindent\emph{Gram matrix reconstruction.}
  From $G^{(j)} = \bm{X}^{(j)}(\bm{X}^{(j)})^T$, choose a
  position matrix $\bm{X}^{(j)}$ (determined up to
  right-multiplication by $R \in O(3)$).  Row~$k$ gives a
  vector $\bm{u}_k$ with
  $\Delta\bm{r}_k^{(j)} = R\,\bm{u}_k$.  Similarly, from
  $G^{(k_a)}$, choose $\bm{X}^{(k_a)}$ (up to
  $R_{k_a} \in O(3)$).
  Throughout, the center of each frame is assigned the zero vector, so that Gram entries carrying the index of the center vanish.

\smallskip\noindent\emph{Frame alignment via shared atoms.}
  By the overlap condition,
  $\calN[j] \cap \calN[k_a]$ contains atoms
  $m_1, m_2, m_3$ whose displacements relative to~$k_a$ are linearly independent.
  Each shared atom~$m$
  satisfies
  $R_{k_a}\,\bm{v}_m^{(k_a)} = R(\bm{u}_m - \bm{u}_{k_a})$.
  Define $V_{k_a}$ and $U_{k_a}$ as the $3 \times 3$ matrices
  with columns $\bm{v}_{m_i}^{(k_a)}$ and
  $\bm{u}_{m_i} - \bm{u}_{k_a}$, respectively.
  Linear independence implies $V_{k_a}$ is non-singular.  Then
  \begin{equation}\label{eq:frame-align}
    T_{k_a} \equiv R^{-1}\,R_{k_a}
    = U_{k_a}\,V_{k_a}^{-1}.
  \end{equation}

\smallskip\noindent\emph{Expansion coefficients.}
  Define $\bm{c}_a = V_{k_a}^{-1}\,\bm{v}_a^{(k_a)} \in \RR^3$,
  so that the position of~$a$ in $j$'s frame is
  $\bm{u}_a = \bm{u}_{k_a} + U_{k_a}\,\bm{c}_a$.
  The coefficients $\bm{c}_a$ satisfy
  the normal equations
  \begin{equation}\label{eq:normal-eq}
    (V_{k_a}^T V_{k_a})\,\bm{c}_a
    = V_{k_a}^T\,\bm{v}_a^{(k_a)},
  \end{equation}
  where $(V_{k_a}^T V_{k_a})_{ij} = G^{(k_a)}_{m_i,\, m_j}$
  and $(V_{k_a}^T \bm{v}_a^{(k_a)})_i = G^{(k_a)}_{m_i,\, a}$
  are entries of~$G^{(k_a)}$.  Since $V_{k_a}$ is non-singular,
  \[
    \bm{c}_a
    = \bigl(G^{(k_a)}_{m_i,\,m_j}\bigr)^{-1}\,
      \bigl(G^{(k_a)}_{m_i,\,a}\bigr),
  \]
  a continuous function of Gram entries of~$G^{(k_a)}$.
  Similarly for $\bm{c}_b$ via~$G^{(k_b)}$.

\smallskip\noindent\emph{Inner product formula.}
  Expanding $\bm{u}_a \cdot \bm{u}_b
  = (\bm{u}_{k_a} + U_{k_a}\,\bm{c}_a)
    \cdot (\bm{u}_{k_b} + U_{k_b}\,\bm{c}_b)$ and using
  $\bm{u}_\alpha \cdot \bm{u}_\beta = G^{(j)}_{\alpha,\beta}$:
  \begin{equation}\label{eq:gram-case2-expand}
    \bar{G}^{(j)}_{ab}
    = G^{(j)}_{k_a,\,k_b}
    + \bm{c}_a^T\,\bm{g}_{k_a \to k_b}
    + \bm{c}_b^T\,\bm{g}_{k_b \to k_a}
    + \bm{c}_a^T\,M_{k_a,\,k_b}\,\bm{c}_b,
  \end{equation}
  where
  \begin{align}
    (\bm{g}_{k_a \to k_b})_i
      &= G^{(j)}_{m_i,\,k_b} - G^{(j)}_{k_a,\,k_b},
    \label{eq:gvec-def} \\
    (M_{k_a,\,k_b})_{ij}
      &= G^{(j)}_{m_i,\,m_j'}
       - G^{(j)}_{m_i,\,k_b}
       - G^{(j)}_{k_a,\,m_j'}
       + G^{(j)}_{k_a,\,k_b}.
    \label{eq:Mmat-def}
  \end{align}
  Here $m'_1, m'_2, m'_3$ denote the shared atoms chosen for the pair $(j, k_b)$, in analogy with the $m_i$ for $(j, k_a)$.
  Every term is a continuous function of entries of~$\calI_j$.

\smallskip\noindent\emph{Independence of the choice of shared atoms.}
The coefficients $\bm{c}_a$, $\bm{g}_{k_a \to k_b}$, and
$M_{k_a, k_b}$ all depend on the choice of shared atoms
$m_1, m_2, m_3$.  However, the final value
$\bar{G}^{(j)}_{ab}$ does not: since
$\bm{u}_a = \bm{u}_{k_a} + U_{k_a}\,\bm{c}_a$ and
$T_{k_a} = U_{k_a} V_{k_a}^{-1}$ is the unique linear map
sending $\bm{v}_m^{(k_a)} \mapsto \bm{u}_m - \bm{u}_{k_a}$
for \emph{all} shared atoms~$m$, the product
$U_{k_a}\,\bm{c}_a = T_{k_a}\,\bm{v}_a^{(k_a)}$ is
independent of which triple defines $V_{k_a}$ and $U_{k_a}$.
Hence $\bm{u}_a$ is the same for any valid triple, and
$\bar{G}^{(j)}_{ab} = \bm{u}_a \cdot \bm{u}_b$ is
well-defined.

In both cases, $\bar{G}^{(j)}_{ab}$ is a well-defined
continuous function of~$\calI_j$.  Since this holds for every
pair $(a, b)$, the entire extended Gram matrix $\bar{G}^{(j)}$
is uniquely and continuously determined by~$\calI_j$.
\end{proof}

We are ready to prove the multi-layer theorem.
The argument proceeds by induction on $L$. At each layer, \cref{lem:arch-express} provides the required approximation capabilities, while \cref{lem:info-suff} ensures that the information reaching that layer is sufficient to determine the extended environment it must represent.
The induction then composes these layerwise approximations to establish the full $L$-layer completeness result.
\begin{theorem}[Multi-layer HGNN approximates
  a complete representation]
\label{thm:multi-hgnn}
Let the $L$-layer HGNN have equal cutoffs $r_c^{(3)} = r_c^{(2)} = r_c$ and sufficient feature dimensions $d_a$ and~$d$.
Then there is a continuous complete representation $\Phi^{*,(L)}$ on $\calE^{[r_c, L]}$ such that for every compact $K \subset \calE^{[r_c, L]}$ consisting of generic environments that satisfy the overlap condition, and every $\delta > 0$, there exist MLP weights for which the $L$-layer HGNN representation $\Phi^{(L)} : \calD_j \mapsto \bm{n}_j^{(L)}$ satisfies $\|\Phi^{(L)} - \Phi^{*,(L)}\|_{C^0(K)} < \delta$.
\end{theorem}

\begin{proof}
By induction on~$L$.  We prove the statement $P(L)$:
\textit{there is a continuous complete representation $\Phi^{*,(L)}$ on $\calE^{[r_c, L]}$ such that for every compact $K$ as in the statement of the theorem and every $\delta > 0$, there exist MLP weights for the $L$-layer HGNN with $\|\Phi^{(L)} - \Phi^{*,(L)}\|_{C^0(K)} < \delta$.
}

The overlap condition gives every center in such a~$K$ at least two neighbors, since the pair formed with one neighbor has a spanning triple containing at least two atoms of~$\calN(j)$.

Base case $P(1)$ follows from \cref{thm:single-hgnn}.

Assume $P(L)$ holds.  Fix the weights for layers $1, \ldots, L$
as guaranteed by~$P(L)$, at an accuracy that may be taken as small as needed.  We find weights for layer~$L+1$.

For $k \in \calN(j)$, genericity keeps every interatomic distance away from~$r_c$, so the hop structure is locally constant and the sub-environment $\calD_k^{[r_c, L]}$ depends continuously on~$\calD_j^{[r_c, L+1]}$.
The set $K_L$ of all sub-environments arising from~$K$ is therefore compact.
Genericity passes to sub-environments.
The overlap condition passes as well.
An atom of hop depth at most $L - 1$ in a sub-environment $\calD_k^{[r_c, L]}$ has hop depth at most $L$ in $\calD_j^{[r_c, L+1]}$, so the overlap condition of the parent applies to its pairs, and the shared neighbors thus provided lie within one hop of that atom, hence at depth at most $L$ and inside the sub-environment.
Each sub-environment center~$k$ retains its neighbor~$j$, as \cref{def:overlap-env} requires.
Hence $K_L$ qualifies for~$P(L)$.

Let $\calW^{(L)} = \calV \times (\RR^d)^n$ denote the
labeled data space for layer~$L+1$, and define the map
$\nu \colon \calD_j^{[r_c, L+1]} \mapsto
\calW_j = (G^{(j)},
(\Phi^{*,(L)}(\calD_1^{[r_c, L]}), \ldots, \Phi^{*,(L)}(\calD_n^{[r_c, L]})))
\in \calW^{(L)}$,
which labels each neighbor with the exact complete representation of its $L$-hop environment supplied by~$P(L)$.
The node features computed by the first $L$ layers approximate these labels, and we account for the difference at the end of the proof.

We first show that $\nu$ is injective modulo~$G$ on the environments of~$K$.
Since $\calN^{L+1}(j) = \bigl(\bigcup_{k \in \calN(j)}
(\{k\} \cup \calN^L(k))\bigr) \setminus \{j\}$, every atom in the $(L+1)$-hop
environment appears in some $k$'s $L$-hop environment.
The $k$-th label is a complete invariant of $\calD_k^{[r_c, L]}$, so it determines the Gram matrix~$G^{(k)}$ and the atom types of that environment up to a relabeling of its atoms.
Genericity fixes the relabeling where it matters.
For $k \in \calN(j)$, the distances from~$k$ to its neighbors are pairwise distinct, and the distance from~$k$ to any atom of~$\calN[j]$ is computable from~$G^{(j)}$.
Hence every atom of $\calN[j] \cap \calN[k]$, and in particular the center~$j$, occupies a determined index in~$G^{(k)}$.
The labeled data~$\calW_j$ therefore determines $\calI_j = (G^{(j)}, \{G^{(k)}\}_{k \in \calN(j)})$ together with these identifications.
Suppose two environments of~$K$ have the same image under~$\nu$ up to~$\Gamma$.
Then they have the same collection~$\calI_j$ of Gram matrices, so by \cref{lem:info-suff} they have the same extended Gram matrix $\bar{G}^{(j)}$, and the labels give their atoms the same types.
By \cref{lem:gram-recon} the positions then agree up to an orthogonal transformation, so the two environments are symmetry-equivalent.

This lets the target representation be pulled back to the labeled data.
By \cref{thm:existence} applied at cutoff $(L+1)r_c$, a continuous complete representation exists on $\calE^{((L+1)r_c)}$.
Since $\calE^{[r_c, L+1]}$ is a $G$-invariant subset, its restriction $\Phi^{*,(L+1)}$ is continuous and complete there.
Because $\nu$ is injective modulo~$G$ on~$K$, $\Phi^{*,(L+1)}$ factors on~$K$ as $\Phi^{*,(L+1)} = g \circ \nu$ for a continuous $\Gamma$-invariant function~$g$ on~$\calW^{(L)}$, by the same factorization and Tietze extension argument used in the proof of \cref{lem:arch-express}.

It remains to approximate~$g$ with a single layer.
As in the proof of \cref{thm:single-hgnn}, it suffices to work on a fixed stratum.
The cross-stratum case follows from the label separation established below.
We verify the hypotheses of \cref{lem:arch-express} for the labels $(\Phi^{*,(L)}(\calD_k^{[r_c, L]}), d_{jk})$, $k \in \calN(j)$.
Let $k \neq k'$ in $\calN(j)$.
If $z_k = z_{k'}$, genericity gives $d_{jk} \neq d_{jk'}$, so the pairs differ in their second component.
If $z_k \neq z_{k'}$, the environments $\calD_k^{[r_c, L]}$ and $\calD_{k'}^{[r_c, L]}$ have different center types, hence lie in different strata and are not symmetry-equivalent, so the complete representation $\Phi^{*,(L)}$ separates them.
The separation $\sigma = \min \|\Phi^{*,(L)}(\calD_k^{[r_c, L]}) - \Phi^{*,(L)}(\calD_{k'}^{[r_c, L]})\|$, taken over $K$ and over the finitely many pairs with $z_k \neq z_{k'}$, is positive because $K$ is compact and $\Phi^{*,(L)}$ is continuous.
The labels of neighbors of different types therefore lie in disjoint sets, and in either case the pairs are distinct.
Both hypotheses of \cref{lem:arch-express} are thus satisfied, so there exist MLP weights for layer~$L+1$ whose realized map~$\tilde{g}$ approximates~$g$ to within~$\delta/2$ on~$\nu(K)$.

Finally, we account for the difference between the exact labels and the computed node features.
The realized map~$\tilde{g}$ is uniformly continuous on a compact neighborhood of~$\nu(K)$, so there is an $\epsilon > 0$ such that inputs within~$\epsilon$ of each other produce outputs within~$\delta/2$.
The exact labels and the computed features differ only in the feature slots, by at most $\max_{k \in \calN(j)} \|\bm{n}_k^{(L)} - \Phi^{*,(L)}(\calD_k^{[r_c, L]})\|$, which is the accuracy of the first $L$ layers on~$K_L$.
Invoking $P(L)$ on~$K_L$ at accuracy~$\epsilon$ therefore keeps the layer-$(L+1)$ output on the computed features within~$\delta/2$ of $\tilde{g} \circ \nu$.
Hence $\Phi^{(L+1)}$ approximates $g \circ \nu = \Phi^{*,(L+1)}$ to within~$\delta$ on~$K$, establishing $P(L+1)$.

By induction, $P(L)$ holds for all $L \ge 1$.
\end{proof}

The induction shows that suitable weights exist but does not say what the network computes at each layer.
For the atoms that share a neighbor's environment this can be made explicit, and the resulting formula shows how Gram information propagates one hop at a time.

\begin{remark}[Constructive Gram propagation]
\label{rem:constructive}
The Case~1 formula~\eqref{eq:gram-case1} provides constructive
insight into how the HGNN extends Gram information layer by layer.
For atoms $a, b$ in a common neighbor's $L$-hop environment
$\calD_k^{[r_c, L]}$, the extended Gram entry
$\bar{G}^{(j)}_{ab} = G^{(k)}_{ab} - G^{(k)}_{aj} -
G^{(k)}_{bj} + d_{jk}^2$
involves only entries of $k$'s Gram matrix and $d_{jk}$.  The HGNN
can evaluate this through~$\rho_3$, since by~$P(L)$ the node
feature~$\bm{n}_k^{(L)}$ approximately encodes a complete
invariant of~$\calD_k^{[r_c, L]}$, from which any continuous
function of~$G^{(k)}$ can be extracted by MLP universality.
Setting $b = a$ yields
$d_{ja}^2 = G^{(k)}_{aa} - 2G^{(k)}_{aj} + d_{jk}^2$,
showing that distances to all $(L+1)$-hop atoms are
constructively determined.
The cross-environment Case~2 requires the universality argument
(\cref{lem:arch-express}) rather than an explicit formula.
\end{remark}

\section{Universal Approximation Theorem}
\label{sec:uat}

\cref{thm:multi-hgnn} supplies a complete representation of the $L$-hop environment, and \cref{thm:completeness-uat} turns completeness into universal approximation. A combination of two results gives the main theorem.
\begin{theorem}[UAT for HGNN]
\label{thm:uat}
Let the $L$-layer HGNN have equal cutoffs $r_c^{(3)} = r_c^{(2)} = r_c$ and sufficient feature dimensions and MLP capacity per layer.
Let $K \subset \calE^{[r_c, L]}$ be a compact set of generic environments that satisfy the overlap condition, each drawn from a configuration satisfying the connectivity condition, with $L$ large enough that $\calD_j^{[r_c, L]} \supseteq \calD_j^{(R_c)}$.
Then for every $\varepsilon \in \calF_{R_c}^0$ and every $\delta > 0$ there exist parameters~$\theta$ with $\|f_\theta \circ \Phi^{(L)} - \varepsilon\|_{C^0(K)} < \delta$.
\end{theorem}

\begin{proof}
Let $\Phi^{*,(L)}$ be the complete representation of \cref{thm:multi-hgnn}, whose hypotheses on~$K$ are part of the present assumptions.
The target~$\varepsilon$ depends only on the neighbors within~$R_c$, so it defines a $G$-invariant function on~$K$, continuous by condition~3 of \cref{def:target-space}.
By completeness of $\Phi^{*,(L)}$ and the argument of the forward implication of \cref{thm:completeness-uat}, applied on $\calE^{[r_c, L]}$ in place of $\calE^{(R_c)}$, there exists a readout~$f_\theta$ with $\|f_\theta \circ \Phi^{*,(L)} - \varepsilon\|_{C^0(K)} < \delta/2$.
The readout~$f_\theta$ is uniformly continuous on a compact neighborhood of $\Phi^{*,(L)}(K)$, so there is an $\epsilon > 0$ such that inputs within~$\epsilon$ produce outputs within~$\delta/2$.
By \cref{thm:multi-hgnn} there exist layer weights with $\|\Phi^{(L)} - \Phi^{*,(L)}\|_{C^0(K)} < \epsilon$.
Combining the two bounds gives $\|f_\theta \circ \Phi^{(L)} - \varepsilon\|_{C^0(K)} < \delta$.
\end{proof}

The theorem is stated for the HGNN, which was introduced as a reference architecture. Any network that can reproduce its layer inherits the conclusion. We verify this for DPA3 and CHGNet in the appendices.
\begin{corollary}[UAT for DPA3]
\label{cor:dpa3}
Under the same conditions as \cref{thm:uat},
the $L$-layer DPA3 architecture with
$r_c^{(3)} = r_c^{(2)} = r_c$, sufficient feature
dimension, sufficient MLP capacity, and readout
network~$f_\theta$ approximates every $\varepsilon \in \calF_{R_c}^0$ to arbitrary accuracy on~$K$.
\end{corollary}

The proof simulates one HGNN layer by a single DPA3 layer and is given in \cref{sm:dpa3} of the supplementary materials.

\begin{definition}[Strong genericity]
\label{def:strong-generic}
A local environment $\calD \in \calE^{(r)}$ is \emph{strongly generic} if no two neighbors are equidistant from the center, that is, if $i \neq k$ implies $|\Delta\bm{r}_i| \neq |\Delta\bm{r}_k|$.
An $L$-hop environment is strongly generic if every local environment within it is strongly generic.
\end{definition}

As with \cref{def:generic}, the excluded set is a finite union of zero sets of nonzero polynomials, so the strongly generic environments form an open dense subset of full measure in each stratum.

\begin{corollary}[UAT for CHGNet]
\label{cor:chgnet}
Under the same conditions as \cref{thm:uat},
with the atom graph and the bond graph sharing the cutoff~$r_c$, and with every environment in~$K$ strongly generic in the sense of \cref{def:strong-generic},
the $L$-layer CHGNet architecture of \cref{sm:chgnet}, with sufficient feature
dimension, sufficient MLP capacity, and readout
network~$f_\theta$, approximates every $\varepsilon \in \calF_{R_c}^0$ to arbitrary accuracy on~$K$.
\end{corollary}

The proof follows the same line as that of \cref{cor:dpa3} and is given in \cref{sm:chgnet} of the supplementary materials.
The strengthening of genericity is needed because the bond message function of CHGNet receives only the center's node features, so the proof recovers each neighbor's features from its distance by a lookup, which requires all neighbor distances to be distinct.
ALIGNN shares the same angle $\to$ edge $\to$ node topology but  differs crucially in that it constructs its messages from single gated linear layers rather than MLPs, and the previous simulation argument does not carry over. We shall describe the architecture and explain where the argument fails in \cref{sm:alignn} of the supplementary materials.
%
\section{Discussion and perspective}
\label{sec:discussion}
The multi-layer theory justifies the practice of stacking layers with a per-layer cutoff smaller than the physical interaction range.
The connectivity condition sets how many layers are needed to cover that range, and the overlap condition tells one how small the per-layer cutoff may be, for condensed matter systems typically slightly above the nearest-neighbor distance.

At the architectural level, the results show that an invariant HGNN relying solely on scalar geometric inputs attains universal approximation, so raising the LiGS order beyond $K = 2$, the 3-body angular level, is not necessary for theoretical expressiveness.
It is important to emphasize, however, that this is a statement about expressiveness alone. The HGNN is intended as a theoretical reference architecture rather than a practical proposal; universal approximation, in itself, does not address data efficiency, computational cost, or the optimization landscape.
Its primary role is to serve as a target for simulation: once a practical architecture can reproduce an HGNN layer, universal approximation follows without the need to re-establish the completeness analysis from scratch.

We carry this out for DPA3 and CHGNet in \cref{cor:dpa3} and \cref{cor:chgnet}, respectively.
The proofs identify a concrete architectural criterion for
guaranteed universal approximation, that the message functions in the
aggregation steps must be \emph{MLPs}, which are universal approximators, not single linear layers.
Both architectures satisfy it.
ALIGNN~\cite{choudhary2021alignn},
which shares the same angle $\to$ edge $\to$ node topology
but uses single gated linear layers in its
convolutions, does not (\cref{rem:alignn}).
This distinction has practical implications.
Replacing the single gated linear layers in ALIGNN's convolutions with MLPs would remove the obstruction identified in \cref{rem:alignn}, although establishing universality for the modified architecture would still require a simulation argument along the lines of \cref{sm:chgnet}.

The theorems rely on genericity.
Completeness is valid on the generic environments of \cref{def:generic}, which excludes highly symmetric arrangements where same-type atoms are equidistant from the center. The restriction is in fact sharp.
The Pozdnyakov counterexample~\cite{pozdnyakov2020incompleteness} exhibits four same-species atoms on a circle, all sharing identical multisets of distances and triplet angles, yet possessing distinct Gram matrices. No single-layer 3-body invariant can separate such configurations. In practice, thermal fluctuations break exact symmetries, so sampled configurations are generic with probability one.
Moreover, in the multi-layer HGNN, the node features of earlier layers help break residual equidistance. The remaining case, zero-temperature crystals, remains a fundamental open problem. Switching to equivariant architectures does not resolve the issue, since equivariant GNNs with a fixed tensor degree $l$ also degenerate on symmetric graphs for specific values of $l$~\cite{cen2024hegnn}.
Both families also face hierarchies of degeneracies.
Raising the body order or increasing the tensor degree resolves progressively more symmetric cases, yet counterexamples are known already at low levels, and no finite level has been shown to achieve completeness across all configurations.

Several other directions remain open.
The DeepSets arguments require the angle-feature dimension~$d_a$, which serves as the latent dimension of the aggregation, to grow at least with the number of neighbor triplets, far beyond the typical practical value of~$64$.
Tightening this gap remains an open problem.
The present results are also qualitative rather than quantitative.
Bounding the approximation rate, that is, how the error decays with network size, would bridge the gap between theory and practical model selection. Finally, extending the analysis from $C^0$ to $C^k$ approximation, so as to match the smoothness of the potential energy surface including forces and Hessians, requires a more detailed study of the differentiability of the DeepSets representation. The smooth symmetrization of Pozdnyakov and Ceriotti~\cite{pozdnyakov2023smooth} offers a promising direction in this regard.

\section*{Acknowledgments}
The AI-driven experiments, simulations and model training were performed on the robotic AI-Scientist platform of Chinese Academy of Sciences. Pingbing Ming is supported by the National Key R\&D Program of China (Grant no.~2024YFA1012502) and by National Natural Science Foundation of China (Grant No.~12371438). Han Wang is supported by the National Key R\&D Program of China (Grant No.~2022YFA1004300) and by the National Natural Science Foundation of China (Grants No.~12525113 and No.~12561160120).
The authors used AI tools to assist with the development of proofs, manuscript editing, and reference verification.
All mathematical proofs have been checked by the authors to the best of their ability, and the authors assume responsibility for all content.

\bibliographystyle{siamplain}
\bibliography{references}

%% file: supplement.tex
\ifdefined\arxivbuild\else
\maketitle
\fi

\section{Proof of the DPA3 corollary}
\label{sm:dpa3}

\begin{proof}[Proof of \cref{cor:dpa3}]
It suffices to show that a single DPA3 layer can simulate one HGNN layer.
An $L$-layer HGNN is then simulated by an $L$-layer DPA3, and the result follows from \cref{thm:uat}.
Each simulated layer realizes a map uniformly continuous on compact sets, so an input error entering a layer propagates to a controlled output error.
Since there are finitely many layers, the per-layer simulation accuracies can be chosen so that the final node features are within any prescribed tolerance of the HGNN output on~$K$.

Set $r_c^{(3)} = r_c^{(2)} = r_c$.
Write $\tilde\rho_3$ for
$s_a(d_{jk})\,\rho_3(\bm{n}_j,\, \bm{n}_i,\, \bm{n}_k,\,
d_{ji},\, d_{jk},\, \cos\theta_{ijk})$,
and define the target intermediate function as
\[
  F_{ji} = \sum_{k \in \calN(j) \setminus \{i\}}
  \tilde\rho_3(\bm{n}_j,\, \bm{n}_i,\, \bm{n}_k,\,
  d_{ji},\, d_{jk},\, \cos\theta_{ijk}),
\]
so that the HGNN output~\eqref{eq:aggregation} is
$\psi_3\bigl(\sum_{i \in \calN(j)} s_a(d_{ji})\,
F_{ji}\bigr)$.
Fix a layer index~$l$ and write $\bm{n}^{(l)}$ for its input node features.
We exhibit weights for which the four updates \eqref{eq:dpa3-eself}--\eqref{eq:dpa3-e2n} carry $\bm{n}^{(l)}$ to the output of one HGNN layer.
Throughout, we maintain the invariant that a designated coordinate of $\bm{e}_{ji}^{(l)}$ carries the distance~$d_{ji}$ and a designated coordinate of $\bm{a}_{j;i,k}^{(l)}$ carries $\cos\theta_{ijk}$.
At initialization this holds by choosing distance and angle embeddings that write $d_{ji}$ and $\cos\theta_{ijk}$ into those coordinates.

\smallskip\noindent\emph{Step~1: the self-message loads node features into the edges.}
At initialization, the edge features depend only on distances, $\bm{e}_{ji}^{(0)} = \phi_d(d_{ji})$.
The self-message~\eqref{eq:dpa3-eself} gives
\begin{equation}\label{eq:dpa3-load}
  \bm{e}_{ji}^{(l+\frac{1}{2})} = \bm{e}_{ji}^{(l)}
  + \delta_s^{(1)}\, \phi_s^{(1)}\bigl(\bm{n}_j^{(l)},\,
    \bm{n}_i^{(l)},\, \bm{e}_{ji}^{(l)}\bigr).
\end{equation}
The MLP~$\phi_s^{(1)}$ receives $\bm{e}_{ji}^{(l)}$ along with both endpoint features, so by MLP universality it can approximate $\bm{x} \mapsto (\bm{x} - \bm{e}_{ji}^{(l)})/\delta_s^{(1)}$ for any continuous target~$\bm{x}$.
With sufficiently large feature dimension~$d$, we may choose the weights so that $\bm{e}_{ji}^{(l+\frac{1}{2})}$ encodes $(\bm{n}_j^{(l)}, \bm{n}_i^{(l)}, d_{ji})$ in designated coordinates.

\smallskip\noindent\emph{Step~2: the angle self-message.}
The angle update~\eqref{eq:dpa3-aself} gives
$\bm{a}_{j;i,k}^{(l+1)} = \bm{a}_{j;i,k}^{(l)} + \delta_s^{(2)}\,\phi_s^{(2)}(\bm{e}_{ji}^{(l+\frac{1}{2})},\, \bm{e}_{jk}^{(l+\frac{1}{2})},\, \bm{a}_{j;i,k}^{(l)})$.
We take $\phi_s^{(2)} = 0$, so the angle features pass through unchanged and the invariant is preserved.
By Step~1 and the invariant, the three arguments of $\phi_c^{(2)}$ in the next step then carry all six arguments of~$\tilde\rho_3$.

\smallskip\noindent\emph{Step~3: the angle aggregation produces $F_{ji}$.}
The edge update~\eqref{eq:dpa3-a2e} reads
\[
  \bm{e}_{ji}^{(l+1)} = \bm{e}_{ji}^{(l+\frac{1}{2})}
  + \delta_u^{(2)}\, \phi_u^{(2)}\Bigl(
    \sum_{k \in \calN(j) \setminus \{i\}}
    w_{jik}\, \phi_c^{(2)}\bigl(
      \bm{e}_{ji}^{(l+\frac{1}{2})},\, \bm{e}_{jk}^{(l+\frac{1}{2})},\,
      \bm{a}_{j;i,k}^{(l+1)}\bigr)
  \Bigr).
\]
By Steps~1 and~2 each summand receives all six arguments of~$\tilde\rho_3$ through a continuous encoding, so by MLP universality $\phi_c^{(2)}$ can approximate any continuous function of them.
The prefactor $w_{jik} = s_a^{(3)}(d_{ji})\, s_a^{(3)}(d_{jk})$ is positive and bounded below on~$K$, because genericity keeps all distances below~$r_c$, the set~$K$ is compact, and $s_a^{(3)}$ is positive on $[0, r_c)$.
Letting $\phi_c^{(2)}$ approximate the quotient of the desired summand by~$w_{jik}$, whose arguments are available through the edge features, the product $w_{jik}\,\phi_c^{(2)}$ realizes the summand.
The sum over~$k$ composed with the outer MLP~$\phi_u^{(2)}$ has DeepSets form, so by \cref{lem:deepsets-mlp} it approximates any continuous symmetric function of the multiset $\{(\bm{n}_k^{(l)}, d_{jk}, \cos\theta_{ijk})\}_{k \neq i}$.
The target below also depends on the parameters $(\bm{n}_j^{(l)}, \bm{n}_i^{(l)}, d_{ji})$, which the outer MLP recovers from the sum as follows.
The inner MLP writes the parameters and the constant~$1$ into designated coordinates, the sum accumulates the parameters times the summand count and the count itself, and the quotient of the two returns the parameters.
The count is at least one, because the spanning triple that the overlap condition provides for the pair $(j, i)$ contains an atom of $\calN(j) \setminus \{i\}$.
Choose the target $(\bm{x} - \bm{e}_{ji}^{(l+\frac{1}{2})})/\delta_u^{(2)}$ with $\bm{x}$ encoding $(F_{ji}, d_{ji})$ in designated coordinates, the distance coordinate taken unchanged from~$\bm{e}_{ji}^{(l+\frac{1}{2})}$.
The target is symmetric because $\bm{e}_{ji}^{(l+\frac{1}{2})}$ is constant across the sum over~$k$, so the residual update gives $\bm{e}_{ji}^{(l+1)} \approx \bm{x}$ and the invariant is preserved.

\smallskip\noindent\emph{Step~4: the edge aggregation produces the HGNN output.}
The node update~\eqref{eq:dpa3-e2n} reads
\[
  \bm{n}_j^{(l+1)} = \bm{n}_j^{(l)} + \delta_u^{(1)}\,
  \phi_u^{(1)}\Bigl(\sum_{i \in \calN(j)} w_{ji}\,
  \phi_c^{(1)}\bigl(\bm{n}_j^{(l)},\, \bm{n}_i^{(l)},\,
  \bm{e}_{ji}^{(l+1)}\bigr)\Bigr).
\]
Since $\bm{e}_{ji}^{(l+1)}$ encodes $(F_{ji}, d_{ji})$, the prefactor $w_{ji} = s_a^{(2)}(d_{ji})$ is likewise positive and bounded below on~$K$, and $\phi_c^{(1)}$ may approximate the quotient $s_a(d_{ji})\, F_{ji}/w_{ji}$, so that $w_{ji}\,\phi_c^{(1)}(\ldots) \approx s_a(d_{ji})\, F_{ji}$.
The sum then gives $\sum_i w_{ji}\,\phi_c^{(1)}(\ldots) \approx \sum_i s_a(d_{ji})\, F_{ji}$.
The target of the outer MLP~$\phi_u^{(1)}$ is $(\psi_3(\sum_i s_a(d_{ji})\, F_{ji}) - \bm{n}_j^{(l)})/\delta_u^{(1)}$, which depends on the parameter $\bm{n}_j^{(l)}$ in addition to the aggregate.
The device of Step~3 recovers $\bm{n}_j^{(l)}$ from the sum, whose summand count $|\calN(j)|$ is positive because the overlap condition grants every center a neighbor.
Hence $\bm{n}_j^{(l+1)} \approx \psi_3(\sum_i s_a(d_{ji})\, F_{ji})$, the output of one HGNN layer.

\smallskip\noindent\emph{Multi-layer extension.}
Steps~1--4 consume the layer's input node features $\bm{n}^{(l)}$ together with the geometric data maintained by the invariant, produce the output $\bm{n}^{(l+1)}$, and preserve the invariant, while the self-message reloads the current node features at the start of every layer.
Stacking therefore needs no separate loading layer, and an $L$-layer HGNN is simulated by an $L$-layer DPA3.
\end{proof}

\section{CHGNet}
\label{sm:chgnet}

CHGNet~\cite{deng2023chgnet} uses the same atom graph and bond graph as DPA3 with $K = 2$: atoms as nodes, bonds within~$r_\mathrm{cut}$ as edges, and the line graph~$\calL(G)$ encoding angles.
We formalize the architecture as written in Eq.~(2) of~\cite{deng2023chgnet}, with \emph{directed} edges, so that each bond $\{j,i\}$ carries two features $\bm{e}_{ji}$ and $\bm{e}_{ij}$.\footnote{The released implementation instead shares one feature per bond ($\bm{e}_{ji} = \bm{e}_{ij}$, a single node of the bond graph) and aggregates the angles centered at both endpoints into it through the same message function~$\phi_e$.  Our simulation argument relies on the directed form and does not cover this bidirectional variant.}
CHGNet maintains three feature types: node features
$\bm{n}_j^l$, edge features $\bm{e}_{ji}^l$ on~$G$,
and angle features $\bm{a}_{ijk}^l$ on~$\calL(G)$.
The initial features are element embeddings (nodes),
smooth radial Bessel function expansions (edges), and
Fourier basis expansions of angles.

Each layer updates all three feature types.
The node update aggregates over neighbors:
\begin{equation}\label{eq:chgnet-atom}
  \bm{n}_j^{l+1} = \bm{n}_j^l
  + L_v\Bigl[\sum_{i \in \calN(j)} \bm{e}_{ji}^0 \odot
    \phi_v\bigl(\bm{n}_i^l,\, \bm{n}_j^l,\,
    \bm{e}_{ji}^l\bigr)\Bigr].
\end{equation}
The bond update aggregates over the angles centered at the tail atom~$j$:
\begin{equation}\label{eq:chgnet-bond}
  \bm{e}_{ji}^{l+1} = \bm{e}_{ji}^l
  + L_e\Bigl[
    \sum_{k \in \calN(j) \setminus \{i\}}
    \bm{e}_{ji}^0 \odot \bm{e}_{jk}^0 \odot
    \phi_e\bigl(\bm{e}_{jk}^l,\, \bm{e}_{ji}^l,\,
    \bm{a}_{kji}^l,\,
    \bm{n}_j^{l+1}\bigr)
  \Bigr].
\end{equation}
Both edge arguments of~$\phi_e$ are the copies directed away from the shared center~$j$.
The angle update is:
\begin{equation}\label{eq:chgnet-angle}
  \bm{a}_{ijk}^{l+1} = \bm{a}_{ijk}^l
  + \phi_a\bigl(\bm{e}_{ji}^{l+1},\,
    \bm{e}_{jk}^{l+1},\, \bm{a}_{ijk}^l,\,
    \bm{n}_j^{l+1}\bigr).
\end{equation}
Here $L_v$, $L_e$ are linear layers, $\phi_v$, $\phi_e$,
$\phi_a$ are gated MLPs (MLPs with sigmoid gating on
intermediate activations), and $\odot$ denotes element-wise
multiplication.
The gate $\bm{e}_{ji}^0$ is a trainable linear image of a smooth radial basis expansion of the bond distance, multiplied by a fixed envelope that vanishes at the cutoff.
By default, CHGNet stacks three full interaction blocks followed by one block with only the node update, and uses a feature dimension of $64$.
The total energy is
$E_\theta = \sum_j \mathrm{MLP}(\bm{n}_j^L)$.

\begin{proof}[Proof of \cref{cor:chgnet}]
We show that two CHGNet layers can approximate the HGNN's per-layer computation up to the outer readout~$\psi_3$, which is absorbed into subsequent blocks or the readout network.
As in the proof of \cref{cor:dpa3}, write $\tilde\rho_3$ for $s_a(d_{jk})\,\rho_3(\bm{n}_j, \bm{n}_i, \bm{n}_k, d_{ji}, d_{jk}, \cos\theta_{ijk})$ and $F_{ji} = \sum_{k \in \calN(j) \setminus \{i\}} \tilde\rho_3$, so that the HGNN layer output is $\psi_3\bigl(\sum_{i \in \calN(j)} s_a(d_{ji})\, F_{ji}\bigr)$.

The gates are handled once and for all.
Since $K$ is compact and the displacement vectors range over $B_{r_c} \setminus \{0\}$, all interatomic distances in~$K$ lie in an interval $[\delta, r_c - \eta]$ with $\delta, \eta > 0$.
Choose the trainable coefficients of every gate~$\bm{e}^0$ so that each component equals the first radial basis function times the envelope, which is bounded below on $[\delta, r_c - \eta]$.
Absorbing a gate into a message function then means letting the MLP approximate the quotient of the target by the gate, which is continuous on~$K$.

At initialization, edge features carry only distance
information: $\bm{e}_{ji}^{(0)} = \phi_d(d_{ji})$, and
node features are type embeddings:
$\bm{n}_k^{(0)} = \mathrm{type\_embed}(z_k)$.

\smallskip\noindent\emph{Step~1: node update.}
The node update~\eqref{eq:chgnet-atom} computes
\[
  \bm{n}_j^{(1)} = \bm{n}_j^{(0)}
  + L_v\Bigl[\sum_{m \in \calN(j)}
  \bm{e}_{jm}^{(0)} \odot
  \phi_v\bigl(\bm{n}_m^{(0)},\, \bm{n}_j^{(0)},\,
  \bm{e}_{jm}^{(0)}\bigr)\Bigr],
\]
where $L_v$ is a linear layer and $\phi_v$ is an MLP.
Since $L_v$ is linear, $\bm{n}_j^{(1)}$ alone does not
necessarily encode the full neighbor multiset.
However, $\bm{n}_j^{(1)}$ will serve as input to the
subsequent bond update MLP~$\phi_e$, and it is the
\emph{composition} that provides universality.

\smallskip\noindent\emph{Step~2: bond update approximates
$\tilde\rho_3$ per summand.}
The bond update~\eqref{eq:chgnet-bond} for
the directed edge~$(j,i)$ is computed with $L_e = I$:
\begin{equation}\label{eq:chgnet-load}
  \bm{e}_{ji}^{(1)} = \bm{e}_{ji}^{(0)}
  + \sum_{k \in \calN(j) \setminus \{i\}}
    \bm{e}_{ji}^{(0)} \odot \bm{e}_{jk}^{(0)} \odot
    \phi_e\bigl(\bm{e}_{jk}^{(0)},\, \bm{e}_{ji}^{(0)},\,
    \bm{a}_{kji}^{(0)},\,
    \bm{n}_j^{(1)}\bigr).
\end{equation}
Substituting the expression for~$\bm{n}_j^{(1)}$, we obtain that each
summand has the form
\[
  \bm{e}_{ji}^{(0)} \odot \bm{e}_{jk}^{(0)} \odot
  \phi_e\Bigl(
    \phi_d(d_{jk}),\, \phi_d(d_{ji}),\,
    \bm{a}_{kji}^{(0)},\;
    \bm{n}_j^{(0)} + L_v\bigl[
      \textstyle\sum_{m} \bm{e}_{jm}^{(0)} \odot
      \phi_v(\bm{n}_m^{(0)},\, \bm{n}_j^{(0)},\,
      \bm{e}_{jm}^{(0)})
    \bigr]
  \Bigr).
\]
This is a function of the per-angle data
$(d_{jk}, d_{ji}, \cos\theta_{kji})$ and the multiset
$\calS_j = \{(\bm{n}_m^{(0)}, d_{jm})\}_{m \in \calN(j)}$.
The dependence on~$\calS_j$ enters through
$\bm{n}_j^{(0)} + L_v[\sum_m g_m]$,
where $g_m = \bm{e}_{jm}^{(0)} \odot
\phi_v(\bm{n}_m^{(0)}, \bm{n}_j^{(0)},
\bm{e}_{jm}^{(0)})$.
The composition $\phi_e(\ldots,\;
\bm{n}_j^{(0)} + L_v[\sum_m g_m])$
has the extended DeepSets form.
The inner function~$g_m$ is parametrized by an MLP~$\phi_v$.
With sufficiently large feature dimension, the type embeddings are chosen supported on one coordinate block and the image of~$L_v$ on the complementary block, so the fourth argument of~$\phi_e$ carries the center feature $\bm{n}_j^{(0)}$ and the aggregate $L_v[\sum_m g_m]$ in disjoint blocks.
By \cref{lem:deepsets-mlp}, there exist weights for $\phi_v$ and~$\phi_e$ such that each summand approximates any continuous function of $(d_{jk}, d_{ji}, \cos\theta_{kji})$, the center feature, and the multiset~$\calS_j$, symmetric in the multiset argument.
The lemma applies in this parametrized form because the per-angle data and the center feature enter the outer MLP as additional inputs, and the approximation is uniform on their compact range.

The target per-angle function is
$\tilde\rho_3(\bm{n}_j^{(0)},\,
\bm{n}_i^{(0)},\, \bm{n}_k^{(0)},\, d_{ji},\, d_{jk},\,
\cos\theta_{ijk})$, depending on specific entries $\bm{n}_i^{(0)}$,
$\bm{n}_k^{(0)}$ of the multiset.
Under strong genericity (\cref{def:strong-generic}), all neighbors of~$j$
have distinct distances.  Define the lookup function
$\lambda : \calM_{\le M}(\RR^d \times \RR_{>0})
\times \RR_{>0} \to \RR^d$ by
$\lambda(\calS_j, d) = \bm{n}_m^{(0)}$, where $m$ is the
unique element of~$\calN(j)$ satisfying $d_{jm} = d$.
On strongly generic configurations this is well-defined and symmetric in~$\calS_j$.
On the compact set~$K$ the distances of distinct neighbors are separated by a positive gap, so $\lambda$ is continuous.
Rewriting:
\[
  \tilde\rho_3(\bm{n}_j^{(0)},\,
  \bm{n}_i^{(0)},\, \bm{n}_k^{(0)},\, \ldots)
  = \tilde\rho_3\bigl(\bm{n}_j^{(0)},\,
  \lambda(\calS_j, d_{ji}),\,
  \lambda(\calS_j, d_{jk}),\, \ldots\bigr),
\]
which is a continuous symmetric function of~$\calS_j$
and the per-angle data.
The gate prefactors $\bm{e}_{ji}^{(0)} \odot \bm{e}_{jk}^{(0)} \odot (\cdot)$
are absorbed by the gate convention fixed at the beginning of the proof, and
$\theta_{kji} = \theta_{ijk}$.
Therefore, each summand in~\eqref{eq:chgnet-load}
approximates $\tilde\rho_3$.
Since the sum in~\eqref{eq:chgnet-load} is a direct
linear sum of per-angle terms, we obtain
\[
  \bm{e}_{ji}^{(1)} \approx \phi_d(d_{ji}) +
  \sum_{k \in \calN(j) \setminus \{i\}}
  \tilde\rho_3(\bm{n}_j^{(0)},\,
  \bm{n}_i^{(0)},\, \bm{n}_k^{(0)},\, d_{ji},\,
  d_{jk},\, \cos\theta_{ijk}).
\]
That is, $\bm{e}_{ji}^{(1)} \approx \phi_d(d_{ji}) + F_{ji}$, and the oppositely directed copy stores $\bm{e}_{ij}^{(1)} \approx \phi_d(d_{ji}) + F_{ij}$.
The summands are written into coordinates on which $\phi_d$ vanishes, so $\bm{e}_{ji}^{(1)}$ carries $\phi_d(d_{ji})$ and $F_{ji}$ in disjoint blocks.

\smallskip\noindent\emph{Step~3: node update computes the HGNN
double sum.}
The node update~\eqref{eq:chgnet-atom} with $L_v = I$
computes
\[
  \bm{n}_j^{(2)} = \bm{n}_j^{(1)}
  + \sum_{i \in \calN(j)}
  \bm{e}_{ji}^{(0)} \odot
  \phi_v\bigl(\bm{n}_i^{(1)},\, \bm{n}_j^{(1)},\,
  \bm{e}_{ji}^{(1)}\bigr).
\]
Each summand receives
$\bm{e}_{ji}^{(1)}$, the copy directed away from~$j$, which stores the inner sum~$F_{ji}$ by Step~2.
By MLP universality, $\phi_v$ can extract this inner
sum from $\bm{e}_{ji}^{(1)}$ and multiply by
$s_a(d_{ji})$, where the distance $d_{ji}$ is available
through the retained block $\phi_d(d_{ji})$ of~$\bm{e}_{ji}^{(1)}$, and the gating factor
$\bm{e}_{ji}^{(0)} \odot (\cdot)$ is absorbed by the gate convention.
The sum over~$i$ then gives
\[
  \bm{n}_j^{(2)} \approx \bm{n}_j^{(1)} +
  \sum_{\substack{i, k \in \calN(j) \\ i \neq k}}
  s_a(d_{ji})\,
  \tilde\rho_3(\bm{n}_j^{(0)},\, \bm{n}_i^{(0)},\,
  \bm{n}_k^{(0)},\, d_{ji},\, d_{jk},\,
  \cos\theta_{ijk}),
\]
which is the HGNN double
sum~\eqref{eq:aggregation} \emph{without} the outer
readout~$\psi_3$.

\smallskip\noindent\emph{Step~4: absorbing~$\psi_3$.}
The HGNN layer output is
$\psi_3(\sum_{i,k} s_a(d_{ji})\, \tilde\rho_3)$, but the
CHGNet layer produces $\bm{n}_j^{(1)} + \sum_{i,k}
s_a(d_{ji})\, \tilde\rho_3$.
With sufficiently large feature dimension, the increment of Step~3 is written into coordinates on which $\bm{n}_j^{(1)}$ vanishes, so the double sum is recoverable from~$\bm{n}_j^{(2)}$.
The missing nonlinearity~$\psi_3$ is handled differently in intermediate and final layers.
In an intermediate layer, the next CHGNet block begins with a node update~\eqref{eq:chgnet-atom}, whose MLP~$\phi_v(\bm{n}_m^{(2)}, \bm{n}_j^{(2)}, \bm{e}_{jm}^{(2)})$ receives each neighbor's double-sum feature and the center's own.
Being an MLP, it can apply~$\psi_3$ to these inputs as part of its computation, so $\bm{n}_j^{(3)}$ encodes $\psi_3$-transformed neighbor information, and the subsequent bond update can compute the next layer's~$\tilde\rho_3$ from it by the extended DeepSets argument of Step~2.
In the final layer, the readout network is an MLP applied after the last CHGNet layer, so it can approximate the composition of the extraction of the double sum, $\psi_3$, and the HGNN readout~$f_\theta$.

\smallskip\noindent\emph{Multi-layer extension.}
Each HGNN layer is simulated by two CHGNet layers.
The first performs the node and bond updates of Steps~1--2, computing the per-angle $\tilde\rho_3$ and storing the inner sums in the edge features, absorbing the previous layer's~$\psi_3$ when $l \ge 2$ as in Step~4.
The second performs the node update of Step~3, which accumulates the double sum into the node features.
The angle updates are set to approximate zero, so the angle features retain the initial encoding of $\cos\theta_{ijk}$ at every layer.
The bond update of the second layer of each pair is not used and is suppressed by setting~$L_e = 0$.
Each simulated layer realizes a map uniformly continuous on compact sets, so an input error entering a layer propagates to a controlled output error.
Since there are finitely many layers, the per-layer simulation accuracies can be chosen so that the final node features are within any prescribed tolerance of the HGNN output on~$K$.
Therefore, an $L$-layer HGNN is simulated by $2L$ CHGNet layers, with the final $\psi_3$ absorbed into the readout, and the result follows from~\cref{thm:uat}.
\end{proof}

\section{ALIGNN}
\label{sm:alignn}

ALIGNN~\cite{choudhary2021alignn} operates on two graphs: an \emph{atom graph}~$G$, whose nodes are atoms and whose edges are the bonds joining pairs of atoms within the cutoff~$r_c$, and its \emph{line graph}~$\calL(G)$.\footnote{We describe all three architectures with the distance cutoff used throughout this paper.  The reference implementation of ALIGNN instead joins each atom to its $12$ nearest neighbors within a cutoff of $8$\,\AA.  This affects neither \cref{rem:alignn} nor any statement below, which depend only on the form of the message functions.}
As in \cref{sm:chgnet}, edges are directed, and each bond $\{j,i\}$ carries two features $\bm{e}_{ji}$ and $\bm{e}_{ij}$, matching the reference implementation.
The line graph, see~\cref{fig:architecture}(c), is constructed
by making each directed edge of~$G$
a node in~$\calL(G)$, and connecting two nodes in~$\calL(G)$ if
the head of the first edge is the tail of the second.  The edges
of~$\calL(G)$ thus correspond to bond angles.
ALIGNN maintains three feature types: node features
$\bm{n}_j^l$, edge features $\bm{e}_{ji}^l$ on~$G$,
and angle features $\bm{a}_{ijk}^l$ on~$\calL(G)$.
The initial node features are element embeddings, the initial edge features are radial basis function (RBF) expansions of bond distances, and the initial angle features are RBF expansions of bond-angle cosines.

Each ALIGNN layer performs two sequential edge-gated
graph convolutions: first on~$\calL(G)$ to update the bond
and angle features using angular information, then
on~$G$ to update the node and bond features using
the enriched bonds.
We write the two convolutions out in turn, since the update order determines which features reach the node aggregation.

\smallskip\noindent\emph{1: Line graph convolution.}
The angle features are updated using the two adjacent
bond features, then the bond features are aggregated
with angle-derived gates:
\begin{align}
  \bm{a}_{ijk}^{l} &= \bm{a}_{ijk}^{l-1}
  + \mathrm{SiLU}\bigl(\mathrm{Norm}(
    A_a\, \bm{e}_{ji}^{l-1} + B_a\, \bm{e}_{jk}^{l-1}
    + C_a\, \bm{a}_{ijk}^{l-1})\bigr),
  \label{eq:alignn-angle} \\
  \tilde{\bm{e}}_{ji}^{l} &= \bm{e}_{ji}^{l-1}
  + \mathrm{SiLU}\Bigl(\mathrm{Norm}\bigl(
    W_{s,a}\, \bm{e}_{ji}^{l-1} +
    \sum_{k \in \calN(j) \setminus \{i\}}
    \sigma_{ijk}\, W_{d,a}\, \bm{e}_{jk}^{l-1}
    \bigr)\Bigr),
  \label{eq:alignn-bond}
\end{align}
where $A_a, B_a, C_a, W_{s,a}, W_{d,a}$ are trainable weight matrices, $\mathrm{Norm}$ is a feature normalization layer, $\sigma$ denotes the sigmoid function, and $\mathrm{SiLU}(x) = x\,\sigma(x)$ is the sigmoid linear unit.
All edge features entering these updates are the copies directed away from the center~$j$.
Both convolutions use the same normalized gate: for features $\bm{x}_s$ aggregated over an index set~$S$,
\begin{equation}\label{eq:alignn-gate}
  \sigma[\bm{x}_s] = \frac{\sigma(\bm{x}_s)}
    {\sum_{s' \in S} \sigma(\bm{x}_{s'}) + \epsilon},
\end{equation}
with $\epsilon > 0$ a small constant for numerical stability.
In~\eqref{eq:alignn-bond} the gate is $\sigma_{ijk} = \sigma[\bm{a}_{ijk}^l]$, aggregated over the angles at the bond~$(j,i)$.
In~\eqref{eq:alignn-node2} it is $\sigma_{ji} = \sigma[\bm{e}_{ji}^l]$, aggregated over the neighbors of~$j$.

\smallskip\noindent\emph{2: Atom graph convolution.}
The edge update loads \emph{both} endpoint node features
into the bond, then the node update aggregates over
neighbors with edge-derived gates:
\begin{align}
  \bm{e}_{ji}^{l} &= \tilde{\bm{e}}_{ji}^{l}
  + \mathrm{SiLU}\bigl(\mathrm{Norm}(
    A\, \bm{n}_j^{l-1} + B\, \bm{n}_i^{l-1}
    + C\, \tilde{\bm{e}}_{ji}^{l})\bigr),
  \label{eq:alignn-edge2} \\
  \bm{n}_j^{l} &= \bm{n}_j^{l-1}
  + \mathrm{SiLU}\Bigl(\mathrm{Norm}\bigl(
    W_s\, \bm{n}_j^{l-1} + \sum_{i \in \calN(j)}
    \sigma_{ji}\, W_d\, \bm{n}_i^{l-1}
    \bigr)\Bigr),
  \label{eq:alignn-node2}
\end{align}
where $A, B, C, W_s, W_d$ are trainable weight matrices.
The total energy is $E_\theta = \sum_j
\mathrm{MLP}(\bm{n}_j^L)$.

A key architectural feature is that the edge
update~\eqref{eq:alignn-edge2} takes both endpoint node
features $\bm{n}_i$ and $\bm{n}_j$ directly.

\begin{remark}\label{rem:alignn}
ALIGNN shares the same
angle $\to$ edge $\to$ node topology as DPA3 and CHGNet,
and its edge update~\eqref{eq:alignn-edge2} incorporates
both endpoint node features, analogous to DPA3's
self-message.  However, the simulation proofs above rely
on the message functions being MLPs, which are universal approximators.
This is
needed for the DeepSets universality argument
and for the lookup
decoding in the proof of \cref{cor:chgnet}.  In ALIGNN, the message
functions are single nonlinear layers.
The inner
functions in the aggregation are sigmoid-gated linear maps, whose only nonlinearity is the scalar gate,
and the outer functions are
$\mathrm{SiLU} \circ \mathrm{Norm} \circ \mathrm{linear}$
rather than MLPs.  Consequently,
\cref{lem:deepsets-mlp} does not directly apply to
a single ALIGNN layer, and the HGNN simulation argument
does not carry over.  Whether ALIGNN achieves UAT through
sufficient depth remains
an open question.
We note that replacing these gated linear messages with MLPs would remove this obstruction.
\end{remark}

\ifdefined\arxivbuild\else
\bibliographystyle{siamplain}
\bibliography{references}
\fi